%% file: neurips_2026.tex
\documentclass{article}

 \usepackage[main, final]{neurips_2026}

\usepackage[utf8]{inputenc} 
\usepackage[T1]{fontenc}    
\usepackage{hyperref}       
\usepackage{url}            
\usepackage{booktabs}       
\usepackage{amsfonts}       
\usepackage{nicefrac}       
\usepackage{microtype}      
\usepackage{xcolor}         
\usepackage{multirow}         
\usepackage{amsmath, amssymb}
\usepackage[noend]{algpseudocode}
\usepackage{algorithm}
\usepackage{graphicx}
\usepackage{subcaption}
\usepackage{booktabs}
\usepackage{multirow}
\usepackage{siunitx}
\usepackage{wrapfig}
\usepackage{wrapstuff}
\usepackage{graphicx}
\usepackage{tcolorbox}
\usepackage{listings}
\usepackage{tasks}
\usepackage{amsthm}
\usepackage{fvextra} 
\newtheorem{theorem}{Theorem}
\newtheorem{lemma}[theorem]{Lemma}

\title{Token Time Continuous Diffusion for Language Modeling}

\newcommand{\boldt}{\mathbf{t}}
\newcommand{\tg}{t_g}

\newcommand{\z}{\mathbf{z}}

\newcommand{\x}{\mathbf{x}}

\newcommand{\V}{\mathcal{V}}
\newcommand{\D}{\mathcal{D}}

\author{%
  Parikshit Bansal \thanks{parikshitb52@gmail.com}\\
  UT Austin\\
  \And
  Sujay Sanghavi \\
  UT Austin \\
}

\begin{document}

\maketitle

\begin{abstract}

In this paper we introduce token time continuous diffusion (TTCD), a new diffusion language model which {\em (a)} operates in continuous space, deterministically mapping Gaussian noise to a final token canvas with no further sampling, and crucially {\em (b)} incorporates a new notion of per-token times, with some tokens proceeding from noise to token at a faster rate than others.
Continuous space modeling helps TTCD avoid the parallel sampling of multiple tokens, which is a key source of inaccuracy at high speedups for models that iterate purely in discrete space. The notion of per-token times helps TTCD to better model conditional generation, allows for more sure tokens to proceed at a faster rate, and allows for differentiated inter-token influences during refinement. TTCD outperforms discrete models at high speedups. 
We train a 160M parameter TTCD model on OpenWebText, and then self-distill it; we find that at high speedups we are comparable in unconditional generation quality, and outperform in conditional generation, several existing models of similar size trained, on the same data, and self-distilled. We achieve similar gains in Sudoku solving as well.

\end{abstract}


\input{intro}

\input{related_work}
\input{method}

\input{experiments}

\input{conclusion}

\bibliographystyle{plain}
\bibliography{neurips_2026}


\appendix

\input{appendix.tex}


\newpage
\input{checklist.tex}

\end{document}

%% file: intro.tex
\section{Introduction}

The dominant paradigms for diffusion language models (LMs) involve iterative refinements in discrete token space, either by unmasking from null tokens or denoising from uniform randomness~\cite{austin2021structured}. Such discrete space diffusion LMs~\cite{ye2025dream, nie2025large} achieve state of the art numbers on language model benchmarks competitive with their auto-regressive counterparts. 
But when multiple tokens are denoised per step, the distribution they are sampled from is the product marginal and not the true underlying joint~\cite{wu2025fast}. This problem is especially pronounced in the high speedup (or few-step generation) setting leading to drop in quality of the generated text. We call this the factorization problem.

We develop a method where token embeddings evolve in a continuous space - starting from Gaussian noise (time equal to 0) and then deterministically evolving to clean token embeddings (time equal to 1). Continuous-space evolution of token embeddings avoids the factorization problem in high speedup settings~\cite{lee2026one, dieleman2022continuous}. 
Prior work in continuous-space diffusion LMs have pointed out the well-recognized~\cite{ye2023dinoiser, dieleman2022continuous, pynadath2025candi, lee2026one} problem of sudden transitions in the surety implied by the token embeddings at different noise levels. Specifically, it is observed that for a large range of low noise levels embeddings are decoded with high surety to a token, and for a large range of high noise levels they appear completely un-informative; there is a very short intermediate transition zone (in terms of noise levels) where ``noise becomes data''~\cite{frans2024one, pynadath2025candi}. Existing works attempt to address this issue via customized time warping in training and inference~\cite{ye2023dinoiser, dieleman2022continuous}. The hypothesis underlying our work is that the identities of some tokens are inherently easier to deduce than others, but the imposition of a single global time (aka noise level) prevents this from organically happening. In addition and equally important, the imposition of a global time prevents the algorithm from explicitly differentiating between input prompt and the canvas that needs to be generated. 


Motivated by these problems, we propose \textit{Token Time Continuous Diffusion (TTCD)}. Our method introduces {\bf per-token times}; each token now proceeds from noise to data at different rates, and this also changes its effect on how other tokens denoise. We also separately maintain a nominal global time, which we progress from 0 to 1 in equal steps without warping (and, which is also nominally the average of the per token times).
Additionally, for conditional generation, with per-token times we can assign a fixed time of 1 to all input prompt tokens throughout the denoising process, while evolving the per-token times of the output canvas positions from 0 to 1 (Sec~\ref{subsec:owt_main}). 
Finally, token time also gives the flexibility to control each tokens denoising trajectory, allowing easier (equiv. more sure) tokens to be denoised earlier and harder tokens to be denoised later (Sec~\ref{subsec:sudoku_main}). 
\begin{figure}
    \centering
    \includegraphics[width=\linewidth]{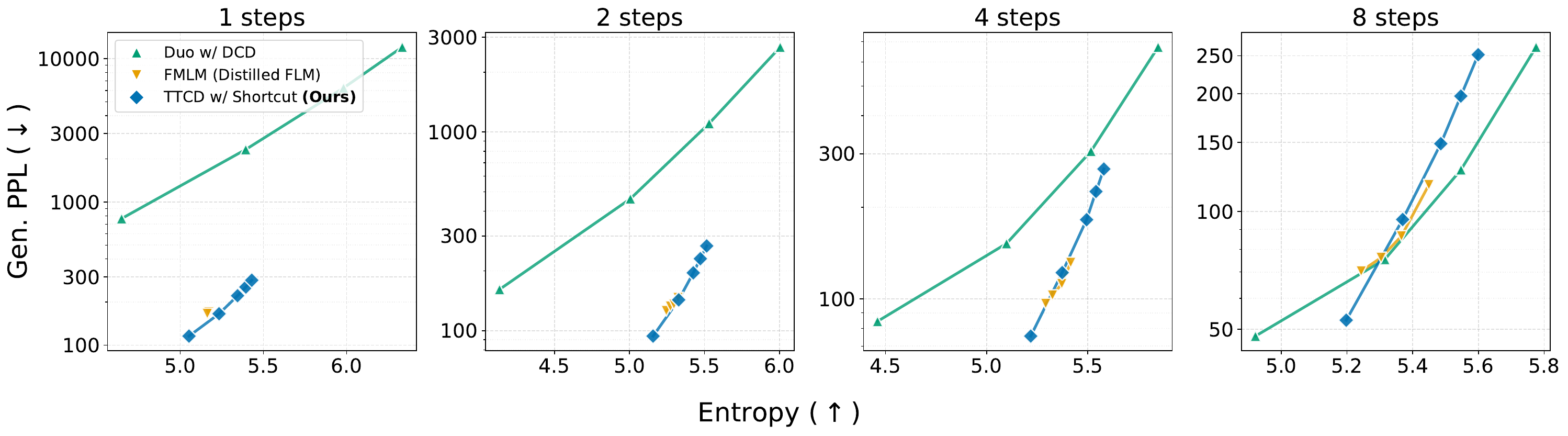}
    \caption{In this figure we evaluate \textbf{unconditional generation} performance of \textbf{distilled diffusion LMs} trained on \textbf{OpenWebText} (see Sec~\ref{subsec:owt_main}). We generate samples from the different methods considered, and plot their generative perplexity (y-axis, lower is better) and text entropy (x-axis, higher is better), for one to eight step generation. 
    Duo w/ DCD is a distilled uniform-noise (discrete) diffusion model and Flow map LM (FMLM) is a distilled continuous diffusion model. Our method is Token-time continuous diffusion (TTCD) distilled via shortcut, and it outperforms (equiv. is more toward the bottom right) the discrete baseline (Duo w/ DCD) for upto four steps, while being competitive with FMLM. Additionally, the results show that even on varying generation parameters, FMLM is restricted to a lower (and smaller) range of text entropy, while TTCD extrapolates to higher entropy.}
    \label{fig:owt_distilled_main}
\end{figure}

\textbf{Contributions:}
\begin{enumerate}
    \item Our work presents a continuous space diffusion language model -- which starts from noise but then deterministically iterates to a final token sequence -- based on the key idea that the noising (and denoising) process should noise (and denoise) tokens at different rates.  
    We formalize this idea by assigning rank variables to each token in the canvas.
    Tokens with higher assigned ranks are denoised earlier (Sec~\ref{subsec:pertokentime}). Ranks are assigned based on per-token entropy after one forward pass (Sec~\ref{subsec:inference}).
    \item We present an architectural change to the widely used adaptive layer norm (adaLN), to allow for flexibility of conditioning the diffusion transformer with per-token times. Importantly, our modification does not introduce new parameters. (Sec~\ref{sec:architecture}).
    \item We show that token-time continuous diffusion (TTCD) achieves 31\% accuracy on solving Sudoku in 2 step generation. Masking based diffusion achieve 11\% on this task and naive continuous-space methods achieve near 0\%(Sec~\ref{subsec:sudoku_main}). We scale TTCD to train a 100M parameter model on OpenWebText. This model (and it's distilled variant) outperforms prior known diffusion LMs for few step generation (Figure~\ref{fig:owt_distilled_main} and  Sec~\ref{subsec:owt_main}).
    \item Further, we show that our model can be easily distilled using self-consistency loss, namely shortcut models for continuous-space flow models (Sec~\ref{sec:distillation_main}). Our distilled model achieves state-of-the-art few step diffusion language model. 
\end{enumerate}

%% file: related_work.tex
\section{Related Work}
There is a rich literature of prior work in continuous-space diffusion for language generation. 
Early works like Diffusion-LM~\cite{li2022diffusion} aim to introduce classifier-based controllable generation at test time for language modeling. These works~\cite{gulrajani2023likelihood} minimized the mean squared error in the embedding space which corresponds to the true ELBO loss. Later work~\cite{dieleman2022continuous} shifted from a mean squared loss to a cross entropy loss with the denoising step parameterized as a convex combination of input embeddings w.r.t. to output token distribution. 
Another line of works explored alternate noising processes, specifically noising in the simplex space, instead of the embedding space, using a Dirichlet distribution (DFM, DDSM)~\cite{,stark2024dirichlet, avdeyev2023dirichlet}. 
Recent works have also explored equipping such probability simplexes with the Fisher-Rao metric ~\cite{davis2024fisher, jocontinuous} to make it a Riemann manifold for getting better ELBO loss. 

From an empirical viewpoint, prior work has also argued for necessity of time warping ~\cite{dieleman2022continuous,ye2023dinoiser, pynadath2025candi,lee2026one}. Time warping methods ensure that the difficulty of the task increases linearly in some metric of interest (e.g., decoding error rate, validation loss, rank of correct token etc). Prior work like CCDD CADD, CANDI~\cite{zhou2025coevolutionary,zheng2025continuously,pynadath2025candi} propose hybrid models which aim to overcome the information loss in the sampling step of discrete models, but are not purely continuous and involve sampling of tokens. 
Finally, prior work has also looked continuous diffusion in a latent embeddings space of language with learnt encoders and decoders ~\cite{lovelace2023latent, lovelace2026stop}
Our work aims to propose a new method for token-space diffusion with emphasis on the necessity of token time.

%% file: method.tex
\section{Method}
\begin{figure*}
\centering
\begin{minipage}{0.45\textwidth}
\centering
\includegraphics[width=\linewidth]{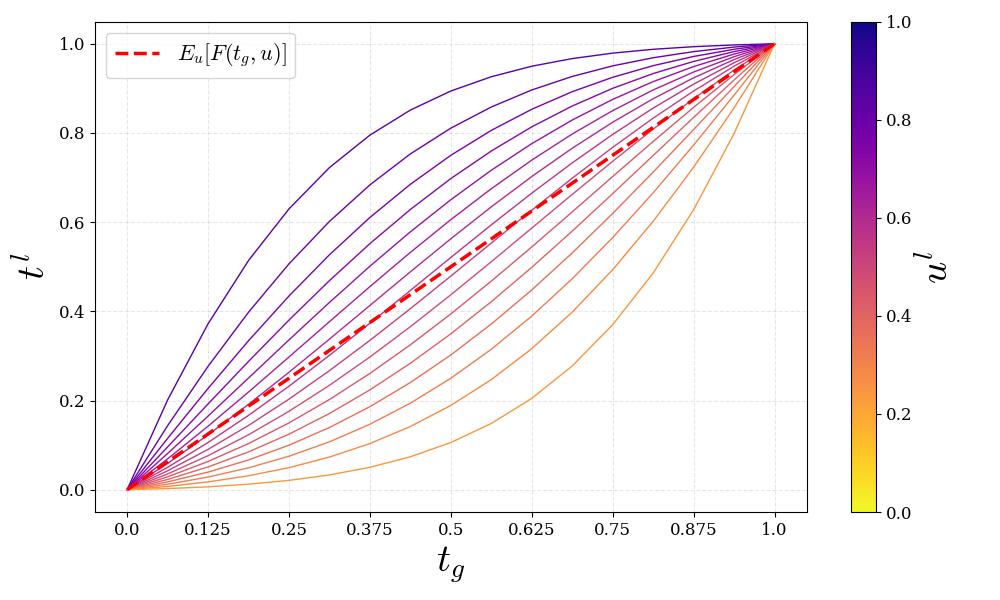}
\end{minipage}
\begin{minipage}{0.52\textwidth}
\centering
\includegraphics[width=\linewidth]{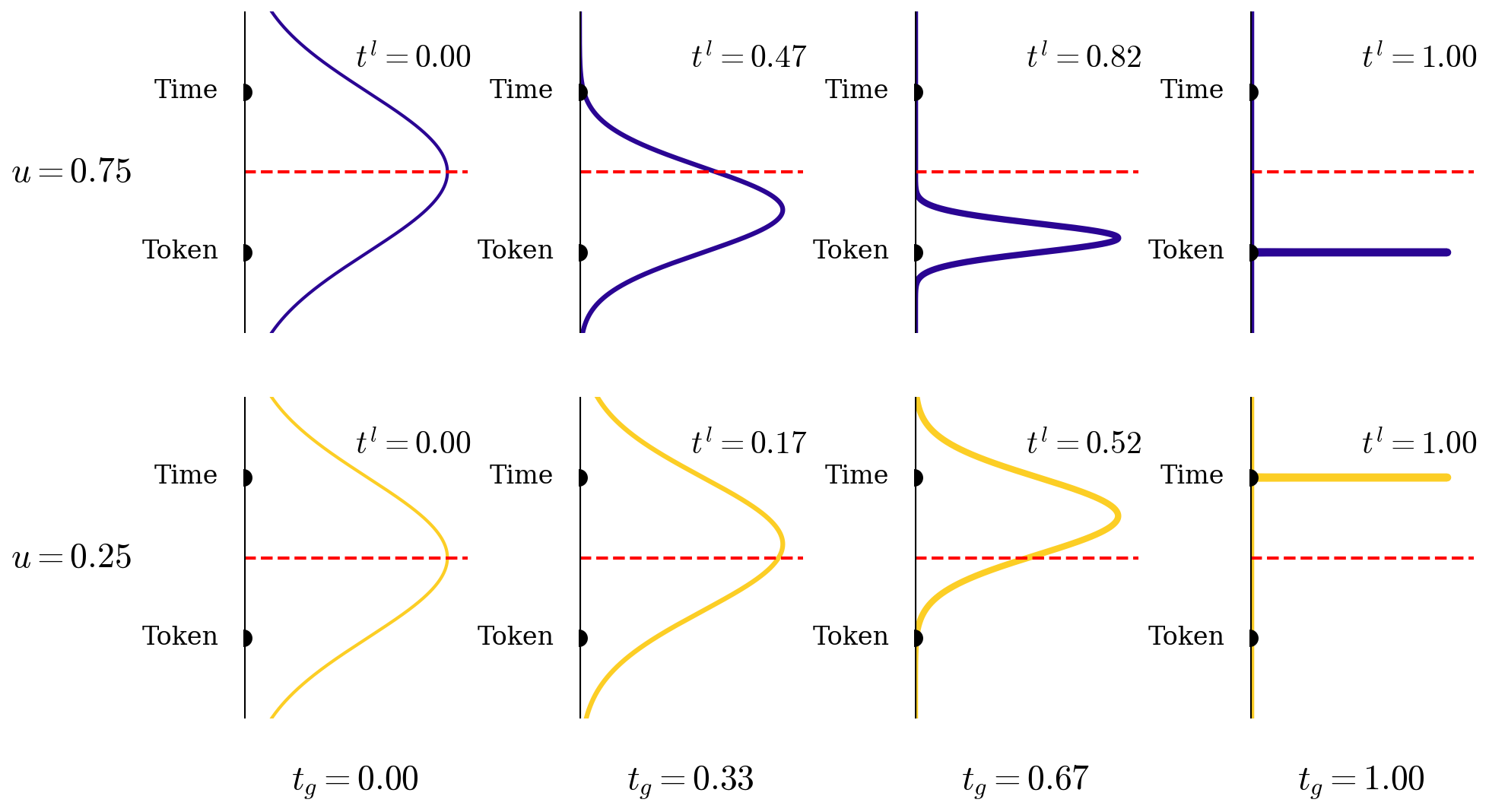}
\end{minipage}
\caption{{\bf Per-token times:}{\em (a) Left}: This shows how local times $t^l$ change as global time $t_g$ varies from 0 to 1, for different ranks $u^l$ (given by colors).\textit{(b) Right}: For a final generated sequence ``token times", this represents the distribution of intermediate vectors $\z$ as a function of global time $t_g$ for given token ranks $u^l$. At $t_g = 0$ they are all the same Gaussian, at $t_g=1$ they are on the respective data token, but in between they interpolate differently based on their ranks.
}
\label{fig:token_time_loss_vs_time}
\end{figure*}

Recall that our task is to develop a model that starts from Gaussian random vectors in $\mathbb{R}^d$, iteratively refines them in continuous space, and in one final step maps them to a string of tokens - such that the resulting distribution of token strings matches the data distribution on which the model is trained.

{\bf Notation.} 
We denote $L$-length token strings as $\x = [x^{1}, x^{2},\ldots, x^{L}]$, and let $\D$ be the
data distribution. Here each token $x^l$ is an element of the vocabulary $\V$. An embedding lookup table $e:\V \rightarrow \mathbb{R}^d$ maps a token $x$ to a vector $z$; we slightly abuse notation and let $e(\x) = [e(x^1),\ldots,e(x^L)]$ denote the set of embeddings for a token string $\x$. We use $\z = [z^1,\ldots,z^L]$ to denote any set of $L$ vectors. Let $q_0$ denote the distribution of $\z$ if each $z^l$ is an independent standard Gaussian $\mathcal{N}(0,I)$ vector in $d$ dimensions; $q_0$ is the ``noise" distribution. Let $[N]$ denote integers between $0$ and $N$. 

{\bf Vanilla flow matching.} Recall that in a basic flow matching setup one considers a data sample $\z_1 = e(\x)$ where $\x \sim \D$, and a noise sample $\z_0 \sim q_0$. Then for any time $t\in (0,1)$ the intermediate sample $\z_t = t\z_1 + (1-t)\z_0$ is made via interpolation. The model is then trained to reconstruct the path to the original data (either $\z_1$, or directly the $\x$ itself) from this intermediate $\z_t$. Note that here all tokens share the same (global) time $t$.

\subsection{Per token times} \label{subsec:pertokentime}

Our approach modifies the above flow matching setup by introducing per-token times; in particular, we now have a global time $\tg$ as well as local per-token times $t^1,\ldots,t^L$. We again choose noise $\z_0 \sim q_0$, and a data string $\x \sim \D$ from which we set $\z_1 = e(\x)$; in addition we also choose a new {\em rank} vector $\mathbf{u} = [u^1,\ldots,u^L]$ where each $u^l$ is chosen iid Unif$(0,1)$. These ranks define the token time. Note that these ranks are chosen only once per data sample, that is, once for the entire noising/denoising path. 

Each token's local time $t^l$ now depends on both the global $t_g$ and its rank $u^l$; in particular
\begin{equation}\label{eqn:local_times}
t^l ~ = ~ F(t_g,u^l) \quad \text{for each token $l$}    
\end{equation}
Each token in the sequence would have it's own rank $u^l$ which gives us the token time $t^l$. Now, at any global time $t_g \in (0,1)$, the intermediate sample is made by interpolating each token using its own local token time, i.e. 
\begin{equation}\label{eqn:noising}
z^l ~ = ~ t^l z^l_1 \, + \, (1-t^l) z^l_0 \quad \text{for each token $l$}
\end{equation}
We detail how we learn a model to reverse the noising process in Sec~\ref{sec:training}.

\textbf{Design choice of $F$.} We would like to select a function $F(\tg,u^l)$ above so that each of the token embeddings interpolates between noise and data; in particular, global time $\tg=0$ corresponds to noise for all, $\tg=1$ corresponds to data for all, and as $\tg$ moves from $0$ to $1$ each token's $t^l$ also monotonically increases (but, different tokens increase at different rates). Additionally, we want tokens with higher ranks to be denoised earlier. This corresponds to higher rank tokens having higher local times for any $\tg$.
Recall also that since the ranks $u^l$ are random Unif$(0,1)$, this means each $t^l$ is a  
random number given $\tg$. We would like to choose $F$ so that the resulting local times are $\tg$ on average. This means we need to choose $F$ so that 
\begin{tasks}(1)
    \task[(A)] $F(0,u) = 0$ and $F(1,u) = 1$ for all $u$. This ensures that all per-token times, independent of rank, start at noise when global $t_g=0$ and all end at data when $t_g=1$.
    \task[(B)] $F(\tg,u)$ monotonically increases in $\tg$ for all $u$. This ensures that the per-token times advance as the global time $t_g$ advances.
    \task[(C)] $F(\tg,u)$ monotonically increases in $u$ for all $t_g$. This ensures that higher rank tokens remain ahead of lower rank tokens as they progress on their journey from noise to final generation. 
    \task[(D)] For $u\sim\mathrm{Unif}(0,1)$, $\mathbb{E}_u[F(\tg,u)] = \tg$ for all $t_g$. This means that the average of per-token times is the global time, at all $t_g$.
\end{tasks}
Our results in this paper correspond to choosing $F$ so that if $u \sim$ Unif$(0,1)$ the resulting random variable $t^l = F(\tg,u)$ follows the distribution Beta$(\frac{1}{1-\tg},\frac{1}{\tg})$. This is formally stated in Lemma~\ref{lemma:choice_of_f} (proof in  Suppl~\ref{subsec:proof_of_lemma}). $F$ is defined as the quantile function of the Beta distribution with appropriate boundary conditions. Figure \ref{fig:token_time_loss_vs_time} (left) depicts example sample paths of how individual per-token times evolve as a function of the global time $\tg$ for different assigned ranks $u^l$. Figure~\ref{fig:token_time_loss_vs_time} (right) shows the distribution of intermediate embeddings as function of global time for a two token string. 
\begin{lemma}\label{lemma:choice_of_f}
Let $F:[0,1]\times[0,1]\to[0,1]$ be defined by
$F(\tg,u) := Q_{\tg}(u)$,
where $Q_{\tg}$ is the quantile function of the $\mathrm{Beta}\!\left(\tfrac{1}{1-\tg},\tfrac{1}{\tg}\right)$ distribution for $\tg\in(0,1)$, with the boundary conventions $F(0,u):=0$ and $F(1,u):=1$. Then $F$ satisfies requirements A,B,C,D above. 
\end{lemma}

{\bf Motivation for per-token times.} We now briefly outline the thought process underlying the design choice of having a per-token time. 
(i) At any global noise level the discrepancy in token times causes a fraction of tokens (those at higher token times) to be more easily deducible compared to others.
This would help model training and inference as the model can learn output token distribution (for the lower token time positions) by conditioning on the sure tokens. Learning inter-token dependency is essential for language modeling~\cite{pynadath2025candi}. 
(ii) Downstream evaluation tasks typically involve a given clean input prompt and a blank noisy canvas. Having only one global sequence level time mis-models this discrepancy.  
(iii) In both language and other domains like Sudoku solving etc, the model may be fundamentally more sure of some tokens and less sure of others. Having per token times provides the model flexibility to incorporate that. Sure tokens can take higher ranks and be denoised earlier. 

\subsection{Training} \label{sec:training}

\begin{figure}[t]
\centering
\begin{minipage}{0.48\textwidth}
\begin{algorithm}[H]
\caption{TTCD Training (Sec~\ref{sec:training})}
\label{algo:training}
\begin{algorithmic}[1]
\State \textbf{Input:} Tokens $\mathbf{x} \sim \mathcal{D}$, noise $q_0$, model $(p_\theta, e_{\theta})$
\Repeat
  \Statex \textsc{\textcolor{orange}{Sample $(\z_0,\z_1)$ for interpolation}}
  \State $\mathbf{z}_1 \gets e_{\theta}(\mathbf{x})$ 
  \State $\mathbf{z}_0 \sim q_0$
  \Statex \textsc{\textcolor{orange}{Sampling global time}}
  \State $t_g \sim \text{Unif}(0,1)$ 
  \Statex \textsc{\textcolor{orange}{Randomly assigning ranks}}
  \State $u^1, \dots, u^L \sim \text{Unif}(0,1)$
  \Statex \textsc{\textcolor{orange}{Computing local times}}
  \State $t^l \gets F(t_g, u^l)$ \Comment{Eqn~\ref{eqn:local_times}}
  \State $z^l \gets t^l z^{l}_1 + (1-t^l) z^{l}_0$ \Statex \textsc{\textcolor{orange}{Cross-entropy loss with $\x$}}
\Comment{Eqn~\ref{eqn:noising}}
  \State $\mathcal{L}(\theta) = - \sum_{l=1}^L \log p^l_\theta ( x^l \mid \mathbf{z}, \mathbf{t} )$
  \State Take gradient step on $\nabla_\theta \mathcal{L}(\theta)$
\Until{converged} 
\State \Return $(p_\theta, e_{\theta})$
\end{algorithmic}
\end{algorithm}
\end{minipage}
\hfill
\begin{minipage}{0.48\textwidth}
\begin{algorithm}[H]
\caption{TTCD Generation (Sec~\ref{subsec:inference})}
\label{algo:inference}
\begin{algorithmic}[1]
\State \textbf{Input:} Input prompt $\x_p$, canvas length $L$, steps budget $N$, model $(p_\theta, e_\theta)$
\Statex \textsc{\textcolor{orange}{Init noise and ranks for canvas }}
\State $\z \sim q_0$
\State $u^1,\ldots,u^L \sim $Unif$(0,1)$
\State $\z_{inp} \gets \texttt{concat}[e_\theta(\x_p);\mathbf{z}]$
\State $\boldt \gets $ 1 for prompt, 0 for canvas  \label{fixing_time}
\Statex \textsc{\textcolor{orange}{Entropy-based rank assignment}}
\State $\eta^l \gets H[p^l_\theta (\cdot \mid \z_{inp}, \boldt)]$ \label{entropy_line}
\State $u^1, \dots, u^L \gets$ Sort ranks by entropy $\eta^l$ \label{sorting_line}
\State $\tg, \Delta \gets 0, 1/N$ (global schedule)
\For{$n = 0, \dots, N-1$}
    \Statex \textsc{\textcolor{orange}{Denoise canvas embeddings}}
    \State Compute updated token times $t^l$ (Eqn~\ref{eqn:local_times})
    \State Compute step-sizes $\Delta^l$ (Eqn~\ref{eqn:delta}) 
    \State Denoise $\z$ using Eqn~\ref{eqn:denoising} 
    \State $\tg \gets \tg + \Delta$
\EndFor
\State $\widehat{\mathbf{x}} \gets \text{argmax } p_\theta(\cdot \mid \z_{inp}, \mathbf{1})$
\State \Return $\widehat{\mathbf{x}}$
\end{algorithmic}
\end{algorithm}
\end{minipage}
\end{figure}

Our model is trained to predict the original string $\x$ given the intermediate vectors $\z = [z^1,\ldots,z^L]$ and set of token times $\boldt = [t^1,\ldots,t^L]$; in particular, our loss function for this string is 
\begin{equation}
\mathcal{L}(\theta) ~ = ~ -\, \sum_{l=1}^L \, \log \, p^l_\theta \left ( x^l \, | \, \z \, , \, \boldt \right )
\end{equation}
Here $p_\theta$ represents our model; $\theta$ are the parameters of the model, and $p_\theta^l$ denotes the distribution it predicts for the $l^{th}$ token. Importantly, the model predictions are conditioned on the token times, a $L$ dimensional vectors, instead of a global time. Conditioning on token times helps model discern the clean tokens from the noisy ones. 
In our implementation, this model is a modification of the diffusion transformer~\cite{peebles2023scalable}; we detail how we made this modification below in Sec~\ref{sec:architecture}. We detail the training algorithm in Algo~\ref{algo:training}. Additionally, following prior work~\cite{dieleman2022continuous} we clamp the maximum L2 norm of the embeddings at one, to prevent their norm from blowing up during cross-entropy based training.

\subsection{Inference} \label{subsec:inference}
\textbf{Denoising step.} Recall that for vanilla language flow models trained with cross-entropy loss on ground truth token string~\cite{davis2024fisher, stark2024dirichlet, jocontinuous}, the flow velocity $v$ is parameterized as the difference between the convex combination of the token embeddings and the current embeddings, scaled by $1-\tg$, where $\tg$ is the global time. Performing the denoising step $z \gets z+\Delta \cdot v$ gives the updated token embeddings for global time $\tg+\Delta$, where $\Delta$ represents the step-size.  
In our approach we model each token as proceeding at a different rate, with their token time determined by their rank. Similarly, a global step-size $\Delta$ corresponds to token step-sizes given as
\begin{equation}\label{eqn:delta}
    \Delta^l = F(\tg+\Delta, u^l) -  F(\tg, u^l)
\end{equation}
where $\tg$ is the global time and $u^l$ is the rank. $\Delta^l$ is non-negative following condition (B) in our design choices for $F$. 
Plugging in token time and token step-sizes into the vanilla denoising step gives us the denoising step for our approach as
\begin{equation}\label{eqn:denoising}
    z^l \gets z^l + \frac{\Delta^l}{1-t^l}\left(\sum_{x \in \mathcal{V}} e_\theta(x) \cdot p^l_\theta(x\mid \mathbf{z} , \boldt) -z^l \right)
\end{equation}
where $p^l_\theta$ is the output probability distribution, $\z=[z^1,\ldots,z^L]$ are the current embeddings, and $\boldt=[t^1\ldots,t^L]$ are token times. 
For prompt-conditioned generation, the token time for the input prompt tokens is $1$ (clean). This implies that the prompt embeddings are never updated and we always pass the ground truth prompt token embeddings to the model. 

\textbf{Assigning ranks.} 
Recall that one of the motivation for token time is to allow for flexibility to denoise sure tokens earlier than unsure tokens. 
In our approach, the rank determines the denoising trajectory of each token with higher ranks corresponding to earlier denoising. Naive inference assigns rank to each canvas token uniformly at random from Unif$(0,1)$. 
Our approach assigns rank to the canvas tokens based on the entropy of their output token distribution. Sure tokens gets higher ranks ensuring that they are denoised earlier and less-sure tokens get denoised later. 
Concretely, suppose we are given an input {\em prompt} $\mathbf{x}_p$ and want to generate an output {\em canvas} $\widehat{\mathbf{x}}$. The algorithm is presented in Algo~\ref{algo:inference}. We elaborate here.\\
{\em (1)} We set and fix token times $t^l = 1$ for all tokens $l$ that are in the input prompt (line~\ref{fixing_time}).\\
{\em (2)} Sample noise $\z_0 \sim q_0$ and ranks  $u^l \sim$ Unif$(0,1)$ for the canvas tokens. Do one forward pass of the model, with input times being $t^l=0$ for canvas (equiv. $\tg{=}0$) and $t^l=1$ for prompt (line~\ref{entropy_line}). \\
{\em (3)} Re-order $u^l$ so that the higher rank corresponds to lower entropy of the $p_{\theta}^l$ distribution (line~\ref{sorting_line}). \\
The ranks are now fixed. We use these re-ordered ranks for {\em all} subsequent denoising steps (i.e. do not reorder after every forward pass). 
Only the entropy of $p^l_\theta$ for completely noised canvas determines the rank of the tokens.  
Using the re-ordered ranks we denoise the string with global step-size of $\Delta=1/N$ where $N$ is our generation step budget. 
Note that every denoising step is deterministic given the initial noise $\z_0$ and ranks $u^l$.


\subsection{Distillation with shortcut models}
\label{sec:distillation_main}
Recall that our work aims to study the high speedup setting where discrete diffusion LMs are prone to generate poor quality of text due to the factorization problem. 
Towards this goal we wish to compare the diffusion LMs (discrete and continuous) particularly designed for few step generation. We now describe a self-consistency~\cite{frans2024one} based distillation for our TTCD model. Importantly, since TTCD is a continuous diffusion model we can directly leverage prior work on building consistency model and progressive distillation methods~\cite{frans2024one, boffi2025build, salimans2022progressive}. 

\textbf{Shortcut Models.} After training TTCD model (following Algo~\ref{algo:training}), we further finetune it with ``shortcut'' loss~\cite{frans2024one}. This is done by (1) conditioning our flow model on the token times at the end of denoising update, in addition to the current token times (2) finetuning the model such that the model predicts the ``average flow velocity'' between the start and the end global times. 
This is algorithmically done by first assuming a distribution of pairs of start global times $t_{start}$ and end global times $t_{end} (\geq t_{start})$. We then train the one step flow model prediction conditioned on $\z, (t_{start}, t_{end})$, capturing the average velocity from $t_{start} \to t_{end}$, with a target two-step average velocity. The two step average velocity is the average of model predictions on first moving from $t_{start} \to t_{mid}$ and then to $t_{mid} \to t_{end}$, where $t_{mid}$ is the middle global time. 

Our TTCD model is trained using a cross-entropy loss. We reformulate the shortcut objective as minimizing the KL-Divergence of one-step output probabilities with two-step unrolled output probabilities. 
Also, to ground shortcut model in ground truth data, we additionally have a cross-entropy loss with $\x$ on infinitesimally small step size $t_{end}{-}t_{start} = \epsilon$~\cite{frans2024one}.
The algo is presented in Algo~\ref{algo:shortcut_training}. 
Importantly, the inference algorithm remains the same except it now includes additional conditioning on the end token times. 

\textbf{Implementation details.} 
For training the shortcut model on OpenWebText (Sec~\ref{subsec:owt_main}) we distill it for 50K steps. For the first 40K training steps we have $t_{start}\sim$Unif$(0,1)$ and $t_{end}\sim$Unif$(t_{start}, 1)$. For last 10K steps, we follow FLM~\cite{lee2026one} and first draw the step size $\Delta=t_{end}{-}t_{start}\sim$Unif$(0,1)$ and then draw $t_{start}\sim$Unif$(0, 1-\Delta)$.
In our experiments we found that additionally conditioning the shortcut flow model on the middle token times (instead of just the start and the end token times) helped stabilize the training process. We believe this is because on taking step-size $\Delta$, the rank information for each token is lost, which is alleviated by additionally conditioning on the middle token times.

\subsection{Architectural changes for TTCD} \label{sec:architecture}
Our model makes an architectural change to how time is incorporated into the DiT transformer architecture~\cite{peebles2023scalable}. DiTs are now a standard choice for diffusion LMs (see e.g.~\cite{sahoo2024simple}). The DiT architecture incorporates a global time, shared by all tokens, via an adaptive layernorm (adaLN) which modulates (scales, shifts, gates) the output of the layernorm operation as a function of the global time~\cite{peebles2023scalable}. TTCD adopts the same set of parameters, but instead of using a global time, each token uses it's own local token-time to modulate outputs of the layernorm. Note that this modification doesn't introduce any additional parameters. For conditioning on multiple token times (e.g., start and end token times) for shortcut distillation (Sec~\ref{sec:distillation_main}), we follow prior work~\cite{frans2024one, lee2026one} and concatenate the time embeddings, before downprojecting them and feeding into the adaLN layer. 


%% file: experiments.tex
\section{Experiments}
In this section, we present empirical evaluations of our proposed method TTCD for the algorithmic task of Sudoku~\cite{kim2025train}, OpenWebText based language generation (where we directly compare to several models also trained on OWT), and finally on molecule generation~\cite{schiff2025simple}. 

\subsection{Sudoku}\label{subsec:sudoku_main}
We consider the task of solving $9{\times}9$ Sudoku puzzle, which has been a popular sandbox to understand discrete diffusion ~\cite{shah2024causal,kim2025train}. Approximately twenty-five cells (out of the eighty-one cells) are filled in the input puzzle, and the generative task is to fill in the rest. We study the accuracy of different discrete diffusion methods at solving this task, i.e. accurately generating a valid sudoku board, at low generation steps of two to sixteen. 

\textbf{Methods.} 
The \textit{discrete masking} method is the sudoku-setting avatar of MDLM; it starts with all empty squares being a MASK token, and all initially filled squares as the ``input prompt'' from which it iteratively unmasks and commits. We consider the two standard unmasking strategies: random, and entropy (i.e. lowest entropy tokens chosen for unmasking).
For \textit{continuous time methods}, we consider noising in the one-hot space of tokens with gaussian noise. Here we ablate two choices for token times, namely "w/ sequence-time" -- where all tokens share a single global $t_g$, and ``w/ token-time" where  each token's one-hot is interpolated to it's local token level of noise determined by the choice of per-token-time function $F$ (see Lemma~\ref{lemma:choice_of_f}). For ``w/ token-time" we show results for random assignment of ranks and entropy based assignment of rank (denoted as ``w/ entropy"). 
Additionally, we also ablate with a decoding error based time warping~\cite{lee2026one} denoted as "w/ warping" for the global-only time method (see Suppl. Fig~\ref{fig:sudoku_warp_fn_app}  for warping function). 
For this setup, we consider input-prompt conditional training where we don't add noise (i.e. mask the token or add gaussian noise) to the given input puzzle. For all methods, we train a small six million parameter transformer model for 100 epochs, with learning rate of 1e-3 and we evaluate the best validation checkpoint.

We note that discrete masking naturally differentiates between sudoku cells deemed to be known and which are unknown, but suffers from the factorization problem of parallel sampling. Conversely, the single global time sequence method does not have a factorization problem, but innately (and inaccurately) deems all tokens (inputs and blanks) to be at the same noise level. Our per-token time method allows for both differentiated treatment of tokens and avoidance of the factorization problem.

\textbf{Results.} We present the percentage of correctly solved Sudoku board in Table~\ref{tab:sudoko_main}. The results show that TTCD is competitive with discrete entropy-based unmasking, demonstrating that TTCD is able to correctly rank the tokens based on their surety at $t_g{=}0$. Furthermore, the discrete baseline's performance at extremely low generation steps degrade due to the factorization error (the model independently samples output tokens). TTCD doesn't commit tokens till the final step and hence is able to outperform the discrete baseline. 

\begin{table}[t]
\centering
\setlength{\tabcolsep}{8pt}
\begin{tabular}{lcccc}
\toprule
\multirow{2}{*}{\textbf{Method}} & \multicolumn{4}{c}{\textbf{Generation Steps}} \\
 & \textbf{2} & \textbf{4} & \textbf{8} & \textbf{16}\\
\midrule
 Discrete masking (random) & {0.62} & {2.73} & {8.79} & {20.09} \\
Discrete masking (entropy) & \underline{11.65} & \textbf{68.29} & \textbf{92.90} & \textbf{97.30} \\
Continuous w/ sequence-time & 0.00 & 0.00 & 3.03 & 20.96 \\
Continuous w/ sequence-time w/ warping & 0.00 & 0.16 & 4.69 & 15.04 \\
Continuous w/ token-time  & {1.01} & {22.27} & {36.75} & {37.89}\\
Continuous w/ token-time w/ entropy (TTCD, \textbf{Ours}) & \textbf{31.51} & \underline{61.33} & \underline{65.85} & \underline{68.46}\\
\bottomrule
\end{tabular}
\caption{{\bf Sudoku 9x9 solving:} We report the percentage of $9{\times}9$ Sudoku boards solved for the experimental setup described in Sec~\ref{subsec:sudoku_main}. The results show that at extremely-low NFE regime of 2 generation steps, token-time is meaningfully better than the baselines considered. At higher generation steps, TTCD performs second best (underlined), just behind discrete masking models with entropy based unmasking at inference. 
}
\label{tab:sudoko_main}
\end{table}

\subsection{Language Models Trained on OpenWebText}\label{subsec:owt_main}
\begin{figure}
    \centering
    \includegraphics[width=\linewidth]{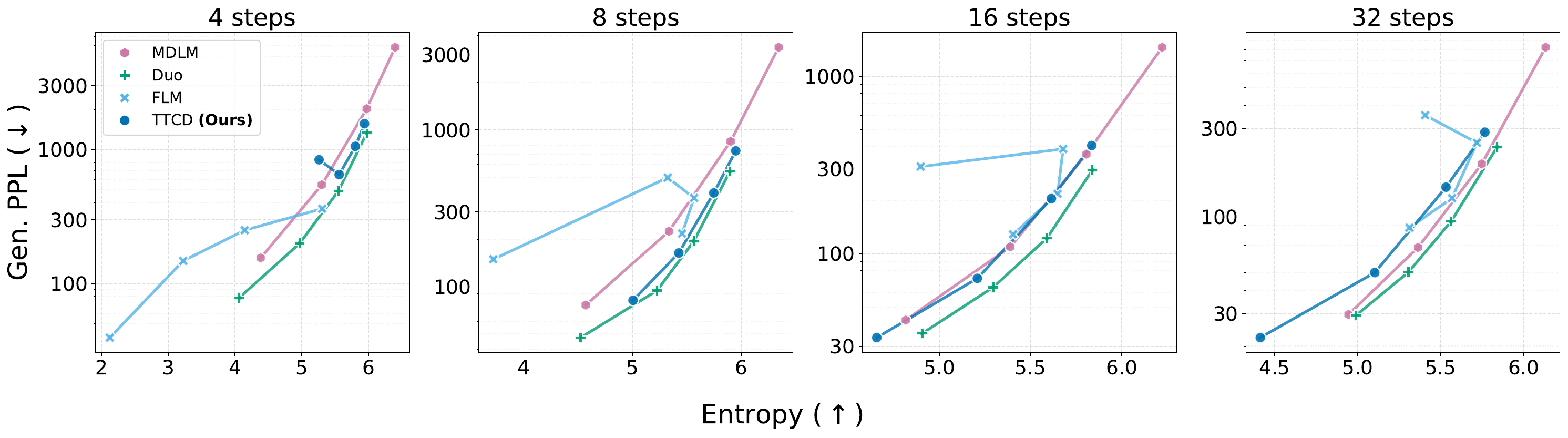}
    \caption{\textbf{Unconditional generation:} In this figure we study the unconditional generation of the different diffusion-based language models trained on the OpenWebText (OWT) dataset (see Sec~\ref{subsec:owt_main}). We compare the perplexity of the generated samples with varying step (forward passes through the diffusion model) budget. The plot shows that the baseline continuous diffusion model's (FLM) entropy on varying it's inference parameters (temperature in the plot) is erratic. For sampling budget less than 16 steps, TTCD closely matched Duo's performance and outperforms MDLM.}
    \label{fig:owt_base_main}
\end{figure}

In this subsection, we present our results on scaling TTCD to 160M parameter model trained on natural language dataset of OpenWebText.

\textbf{Experimental Setup.} We train our TTCD model on the  OpenWebText dataset for 1M steps, following training parameters of prior work~\cite{sahoo2024simple}. 
Given a trained TTCD model, we evaluate it's generative performance under two seperate modes : \textit{unconditional} and \textit{prefix-conditioned}.  Unconditional generation aims to generate the full canvas of 1024 tokens without any prompt conditioning. Prefix-conditioned generation generates the last $K$ (32 and 128 in our experiments) tokens conditioned on a clean prefix (1024-$K$ tokens) of validation set token strings. This evaluation aims to mimic conditional generative performance which is closer to how would be used eventually in downstream tasks. 
For unconditional generation we evaluate the models by first sampling token strings at varying generation step budget. Then the generated strings are using a GPT2-large model, which assigns a perplexity to each generated string (lower is better). Additionally we also report the text entropy of generated string. The entropy captures the diversity of tokens within a string (higher is better). 
For prefix-conditioned generation, generative perplexity and entropy are computed on just the output canvas (conditioned on the input prompts). Qualitative generations for this setup are reported in appendix Sec~\ref{sec:qualitiatve_app}.

\textbf{Methods.} We consider the following baselines. First, for the {\bf non-distilled} models we consider: {\em (a)  MDLM}~\cite{sahoo2024simple} is the masked discrete diffusion model. {\em (b) Duo}~\cite{sahoo2025diffusion} is the uniform-noise discrete diffusion model.
{\em (c) FLM}~\cite{lee2026one} is a recent continuous-space discrete-generation diffusion model, and {\em (d) TTCD \textbf{(ours)}} with unconditional training as described in section \ref{sec:training}. Then, for the {\bf distilled} models (i.e. those that have gone through a second step of further training with some form of self distillation) we consider {\em (a) Duo w/ DCD}~\cite{sahoo2025diffusion}which is the distilled DCD model, {\em (b) FMLM}~\cite{lee2026one} which is the distilled version of FLM, and {\em (c) TTCD (w/ Shortcut) \textbf{(Ours)}} is the distilled version of TTCD as described in Sec~\ref{sec:distillation_main}

\textbf{Gen PPL vs entropy curves.} Prior work~\cite{zheng2024masked, pynadath2025candi} have shown that indexing model performance by just generative perplexity evaluated by an AR model leads to spurious findings. Diffusion LMs often trade-off the entropy of generated text (by repeating words) to achieve a lower generative perplexity. The entropy can be controlled by varying the inference time parameters like (a) for token-prediction models : the softmax temperature for output probabilities $p^l_\theta$, or (b) only for continuous-space models : the initial noise norm, in other words, variance of Gaussian $q_0$. We present our results by varying these inference parameters to get points on the Gen ppl. vs entropy curve. The softmax temperature is chosen from \{0.8, 0.9, 1.0, 1.1\}, while initial noise norm (standard deviation of $q_0$) is from \{0.5, 0.7, 0.9, 1.1\}. A better model is on the bottom right side of the presented curves.

\textbf{Results on unconditional generation} are encapsulated in the caption of Fig~\ref{fig:owt_base_main} (kindly refer to Table~\ref{table:detailed_base_app} for numbers). 
We ablate on choice of parameter (sampling temperature or initial noise norm) to vary for the continuous-space models (FLM, TTCD) in Fig~\ref{fig:owt_noise_v_temp_base_appendix} in the appendix and report the best here.
TTCD based generation is lesser prone to repeating tokens (collapsing entropy) as the generated tokens become sure at different global noise levels. Hence once a token is sure (i.e., higher token time) at a global noise level, it is less likely to be repeated at other token positions still currently at lower token times. 

\textbf{Results on prefix-conditioned generation.} Fig~\ref{fig:owt_conditional32_main} shows evaluations for 32 canvas length.
For TTCD, as detailed in Algo~\ref{algo:inference}, the token-time for the clean prefix is fixed as one. 
The results show that TTCD (w/ shortcut) achieves the best prefix-conditioned generation, demonstrating that inclusion of token-time makes TTCD naturally suited for conditional generation. FLM and Duo condition on one global time which goes from 0 to 1 for both methods. 
FLM performs poorly on the task. The generations from FLM are entropy collapsed (${<}3$ entropy) and out of bounds in four step generation subfigure (see Fig~\ref{fig:owt_conditional_appendix} for complete plots). We report numerical numbers in Sec~\ref{subsec:conditional_app}.

\begin{figure}
    \centering
    \begin{minipage}[t]{0.48\textwidth}
        \vspace{0pt}
        \centering
        \includegraphics[width=\linewidth]{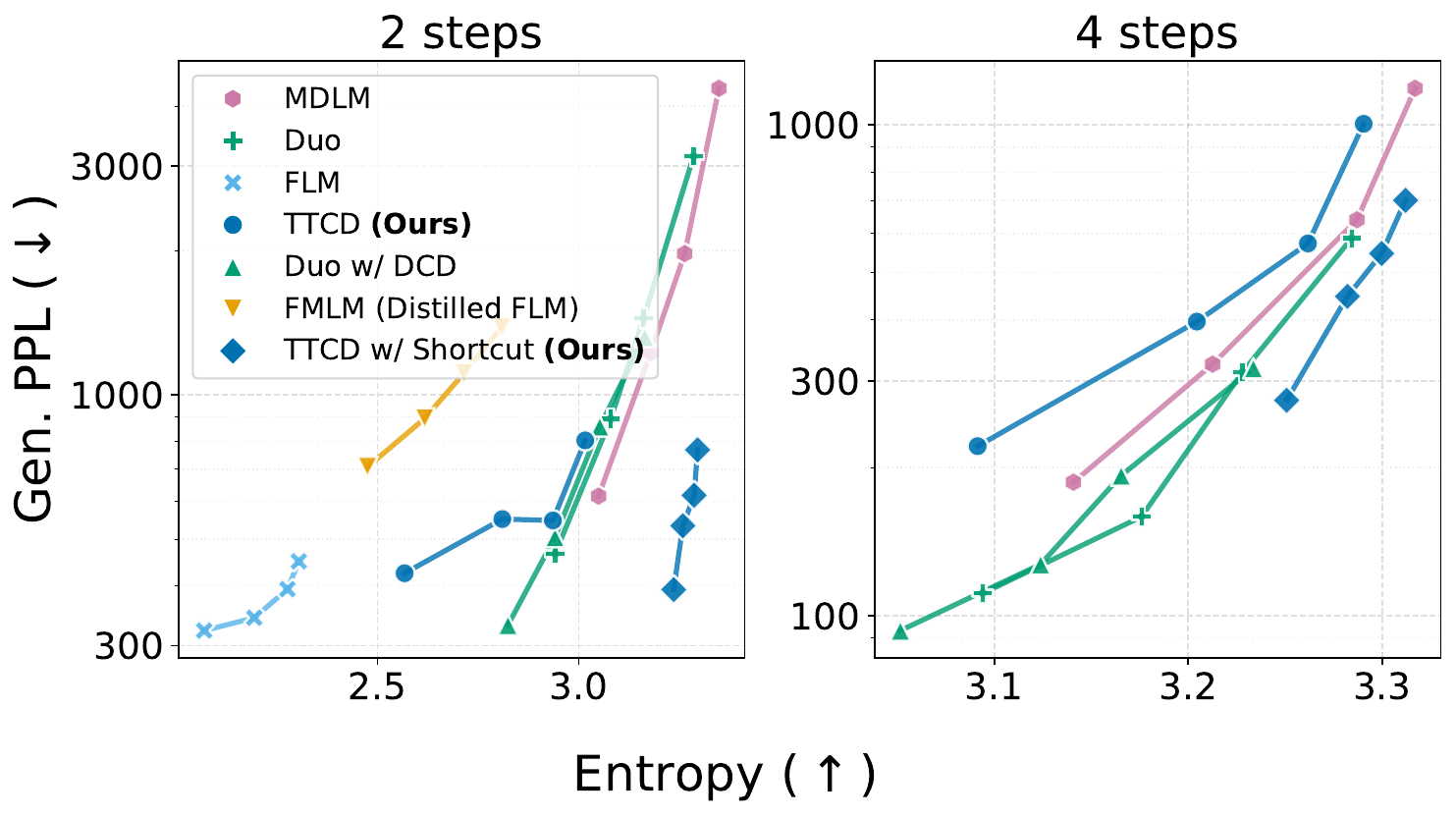}
    \end{minipage}%
    \hfill
    \begin{minipage}[t]{0.48\textwidth}
        \vspace{0pt}
        \footnotesize
        \begin{tcolorbox}[colback=blue!5, colframe=blue!60!black, boxrule=0.4pt, arc=1mm, left=3pt, right=3pt, top=1pt, bottom=1pt, boxsep=1pt]
            \textbf{Prompt:} [omitted]. On a console known for 3D games, Symphony of the Night shined as a model of what 2D could be
        \end{tcolorbox}
        \begin{tcolorbox}[colback=green!5, colframe=green!50!black, boxrule=0.4pt, arc=1mm, left=3pt, right=3pt, top=1pt, bottom=1pt, boxsep=1pt]
        \textbf{Duo w/ DCD:} 
        2D games;Play Theft of II’s Fandals 1978: Return of the Outrider’s Master,<|endoftext|>
        \end{tcolorbox}
        \begin{tcolorbox}[colback=purple!5, colframe=purple!60!black, boxrule=0.4pt, arc=1mm, left=3pt, right=3pt, top=1pt, bottom=1pt, boxsep=1pt]
            \textbf{TTCD w/ Shortcut (Ours):} done two years ago. U Castlevania is able to subscribe to remake of the 3D controllers to upgrade play demos.<|endoftext|>The 3D government imported<|endoftext|>
        \end{tcolorbox}
    \end{minipage}
    \caption{\textbf{Prefix-conditioned generation on canvas of 32 length}. \textit{(a) Left:} We evaluate different models on conditional generation of a canvas of 32 tokens (see Sec~\ref{subsec:owt_main}).
    The plots demonstrate the efficacy of TTCD (w/ shortuct) over the baselines diffusion models for two (16x speedup) and four (8x speedup) generation steps. \textit{(b) Right:} The figure shows a qualitative example of the prefix-conditioned generation for Duo w/ DCD and TTCD (w/ shortcut) at 4 generation steps.}
    \label{fig:owt_conditional32_main}
    \vspace{-2em}
\end{figure}

\begin{wrapstuff}[r, width=0.50\linewidth, type=figure]
    \centering
    \includegraphics[width=0.8\linewidth]{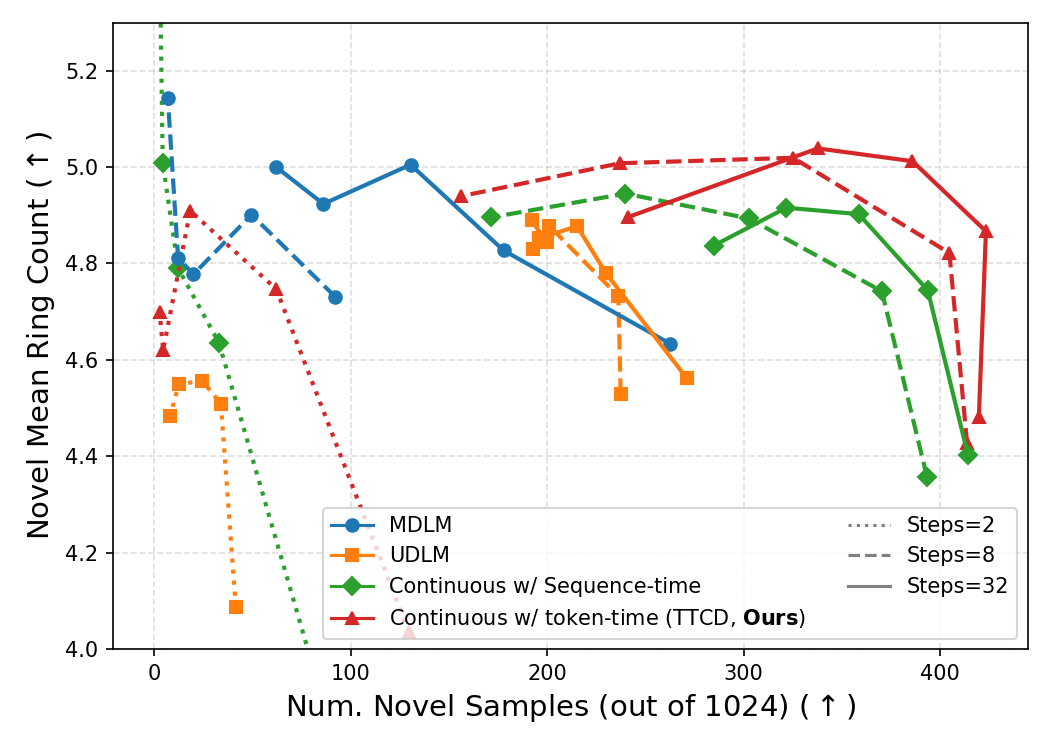}
    \caption{Classifier-free guidance on QM9}
    \label{fig:guidance_ringcount_main}
\end{wrapstuff}

\subsection{Classifier-free guidance}\label{subsec:classifier_main}
We evaluate classifier-free guidance on the molecule data QM9 ~\cite{schiff2025simple}. The methods are trained on the QM9 dataset following prior work set of training hyper parameters ~\cite{schiff2025simple}. The baseline methods are (a) \textit{UDLM} is a discrete uniform noising model and (b) \textit{MDLM} is masked discrete diffusion model. We consider two continuous diffusion baselines (c) \textit{Continuous w/ sequence time} and (d) \textit{Continuous w/ token-time \textbf{(Ours)}} (see Sec~\ref{subsec:sudoku_main} for details on continuous methods).  

\textbf{Results} are reported in Fig~\ref{fig:guidance_ringcount_main}. The plot here shows the mean of ring counts for novel molecules (on the y-axis) against the number of novel molecules out of 1024 generated molecules. Continuous-space diffusion is able to generate more valid molecules compared to discrete baselines on this task. Incorporation of token-time (TTCD) helps push the frontier further.
The different points on curve correspond to different guidance levels varying from one to five. We plot the results for different number of steps (denoted as steps in the figure). 
We also report mean QED numbers on this dataset in appendix Fig~\ref{fig:guidance_qed_appendix}.

%% file: conclusion.tex
\section{Conclusion, Limitations and Future work} 
\label{sec:conclusion}
\textbf{Conclusion}. We presented TTCD, a continuous-state diffusion LM, with the key differentiator being the inclusion of token-time apart from the global sequence level time. The inclusion of token-time helps our method be significantly more robust of inference-time parameters and naturally allows for prefix-conditioned generation. TTCD outperforms the best known discrete method for high speed settings on both unconditional and prefix-conditioned generation. 
\textbf{Limitations.} Our results are limited to 100M parameter scale setting. Scaling TTCD to 1B parameter will help compare to multi-token predictor methods for autoregressive models which is currently missing. 
\textbf{Future work.} The notion of token-time can be extended to uniform noising discrete models. This can potentially help prefix-conditioned for these models. 

%% file: appendix.tex
\input{app_theoretical_results}
\clearpage
\input{app_additional_details}
\clearpage
\input{app_quantitative_results}
\clearpage
\input{app_qualitative_results}
\clearpage
\section{Compute Requirements}\label{app:compute}

\textbf{OpenWebText} All are experiments are done on NVIDIA H100 GPUs. Training the TTCD model on OpenWebText took 2 days on 32 H100s. 

\textbf{Sudoku and QM9} These experiments were done on NVIDIA A40 and needed around 10 hours for each training run on a single A40 GPU. 

%% file: app_theoretical_results.tex
\section{Theoretical Results}
\subsection{Proof of Lemma~\ref{lemma:choice_of_f}}
\label{subsec:proof_of_lemma}

\begin{proof}
For $\tg\in(0,1)$, write
$$
\alpha(\tg) \;:=\; \frac{1}{1-\tg}, \qquad \beta(\tg) \;:=\; \frac{1}{\tg}
$$
so that $F(\tg,\cdot) = Q_{\tg}$ is the quantile function of $\mathrm{Beta}(\alpha,\beta)$ (we omit dependence on $\tg$ when it's apparent from context). We verify the four requirements in turn.

\medskip
\noindent\textit{(A)} The identities $F(0,u)=0$ and $F(1,u)=1$ hold by the assumed boundary conventions. 

We now show that assumed convention is consistent with the continuous limits of the Beta family. The mean and variance of $\mathrm{Beta}(\alpha,\beta)$ are simplified as : 
$$
\mu(\tg)=\frac{\alpha}{\alpha+\beta} = \tg \qquad \sigma^{2}(\tg)=\frac{\alpha\beta}{(\alpha+\beta)^2(\alpha+\beta+1)}=\frac{\tg^2(1-\tg)^2}{\tg(1-\tg)+1}
$$
As $\tg \to 0^+$ we have both $\mu(\tg) \to 0^+$ and $\sigma^2(\tg) \to 0^+$. Since both mean and variance approach $0$, every quantile of the distribution also approaches to 0. Hence, $Q_{\tg}(u)\to 0$ for every $u\in(0,1)$. The limit for $\tg \to 1^-$ can be symmetrically derived.

\medskip
\noindent\textit{(C)} This holds by the fact that the CDF of the Beta distribution is strictly increasing, and hence values at higher quantiles are strictly greater than values at lower quantiles. 

\medskip
\noindent\textit{(D)} For $u\sim\mathrm{Unif}(0,1)$ we know that $Q_{\tg}(u)\sim\mathrm{Beta}(\alpha,\beta)$. This follows from inverse-transform identity. Simply stated, for a distribution, doing CDF (of that distribution) inverse of a uniform random variable gives samples from the distribution. 
Hence $\mathbb{E}_{u}\!\left[F(\tg,u)\right]$ is equal to the mean of the Beta distribution.
Using the formula for the mean of Beta distribution derived for part (A) we have $\mu(\tg) = \tg$.

\medskip
\noindent\textit{(B)} We first consider the likelihood ratio of the two Beta densities,
$$
\frac{f_{\alpha(\tg'),\beta(\tg')}(x)}{f_{\alpha(\tg),\beta(\tg)}(x)} \;\propto\; \frac{x^{\alpha(\tg')-\alpha(\tg)}}{(1-x)^{\beta(\tg)-\beta(\tg')}}, \qquad x\in(0,1).
$$
Since $\alpha$ is strictly increasing and $\beta$ is strictly decreasing in $\tg$, we have $\alpha(\tg')-\alpha(\tg)>0$ and $\beta(\tg)-\beta(\tg')>0$. Hence for increasing $x$, the numerator is increasing and the denominator is decreasing. Hence the likelihood ratio is strictly increasing in $x$. The Beta family thus has a strict monotone likelihood ratio in $\tg$, which implies strict stochastic dominance:
$$
\mathrm{Beta}(\alpha(\tg),\beta(\tg)) \;\prec_{\mathrm{st}}\; \mathrm{Beta}(\alpha(\tg'),\beta(\tg')).
$$
The stochastic dominance property implies that $F_{\tg}(x)>F_{\tg'}(x)$ for every $x\in(0,1)$, where $F_{\tg}$ is the CDF of Beta$(\alpha(\tg),\beta(\tg))$ (we have overridden the notation for $F$ to be both the original function and the CDF).
Since $Q_{\tg} = F^{-1}_{\tg}$, we have $Q_{\tg}(u)<Q_{\tg'}(u)$ and hence $F(\tg,u)<F(\tg',u)$. 

\medskip

Properties (A)--(D) are thus established.
\end{proof}

%% file: app_additional_details.tex
\section{Additional Details}

\subsection{Shortcut Details}
We present the shortcut training algorithm in Algo~\ref{algo:shortcut_training}. The inference is the same as standard inference with extra token time conditioning for $g_\phi$.
\begin{algorithm}[H]
\caption{TTCD Distillation using Shortcut (Sec~\ref{sec:distillation_main})}
\label{algo:shortcut_training}
\begin{algorithmic}[1]
\State \textbf{Input:} Token strings $\x \sim \mathcal{D}$, gaussian noise distribution $q_0$, trained TTCD ($p_\theta$, $e_{\theta}$), time distribution $\mathcal{T}(t_{start},{t_{end}})$.
\Statex \textsc{\textcolor{orange}{${g}_\phi$ and $p_\theta$ have the same set of parameters, but ${g}_\phi$ concats token times}}
\Statex \textsc{\textcolor{orange}{${g}_\phi$ takes concat of three token times for training stability}}
\State  $({g}_\phi, {e}_{\phi}) \gets (p_\theta, e_{\theta})$
\Repeat
  \State $\z_1 \gets e_{\phi}(\x)$ 
  \State $t_{start}, t_{end} \sim \mathcal{T}$ 
  \State $t_{mid} \gets (t_{start}+ t_{end})/2$
  \State $t_{mid,end} \gets (t_{mid}+ t_{end})/2$
  \State $t_{start,mid} \gets (t_{start}+ t_{mid})/2$
  \State $u^1, u^2, \ldots, u^L \sim \text{Unif}(0,1)$ 
\Statex \textsc{\textcolor{orange}{Noise $\z_1$ to $t_{start}$ noise level. See Eqn~\ref{eqn:noising}.}}
  \State $\z \gets \mathtt{NoisingProcess}(\z_1, t_{start}, [u_0,u_1,\ldots,u_L])$ 
\Statex \textsc{\textcolor{orange}{Denoising update to get  $\z_{mid}$ at global time $t_{mid}$. Use $g_\phi$ for denoising in Eqn~\ref{eqn:denoising}.}}
  \State $\z_{mid} \gets \mathtt{DenoisingStep}(\z, t_{start}, t_{mid}, [u_0,u_1,\ldots,u_L])$ 
\Statex \textsc{\textcolor{orange}{Target logits are two step unroll, first got $\z_{mid}$ and then target logits}}
\Statex \textsc{\textcolor{orange}{$\boldt_a$ represent token times with ranks $u^l$ and global time $t_a$}}
  \State $\mathrm{target} \gets g_\phi \left ( \cdot \, | \, \z_{mid} \, , \, (\boldt_{mid}, \boldt_{mid, end}, \boldt_{end}) \right )$
\Statex \textsc{\textcolor{orange}{Consistency loss with stopgrad on target}}
  \State $\mathcal{L}(\phi) ~ = ~ \mathrm{KL} [\texttt{stopgrad} \left(\mathrm{target}\right ) \mid\mid g_\phi \left ( \cdot \, | \, \z \, , \, (\boldt_{start}, \boldt_{mid}, \boldt_{end}) \right )]$
\Statex \textsc{\textcolor{orange}{Token string prediction on infinitesimally small step size}}
  \State $\mathcal{L}(\phi) ~ = ~ \mathcal{L}(\phi) - \sum_{l=1}^L \log g^l_\phi \left ( x^l \, | \, \z \, , \, (\boldt_{start}, \boldt_{start}+\epsilon, \boldt_{start}+2\epsilon) \right )$
  \State Take gradient step on $\nabla_\phi \mathcal{L}(\phi)$
\Until{converged}
\end{algorithmic}
\end{algorithm}

\subsection{Sudoku Experimental Details}

\paragraph{Global-only time continuous time warping} We follow ~\cite{lee2026one} to use decoding error rate as the time warping function. The warping function used is depicted in Fig~\ref{fig:sudoku_warp_fn_app} 

\begin{figure}[h]
    \centering
    \includegraphics[width=0.50\linewidth]{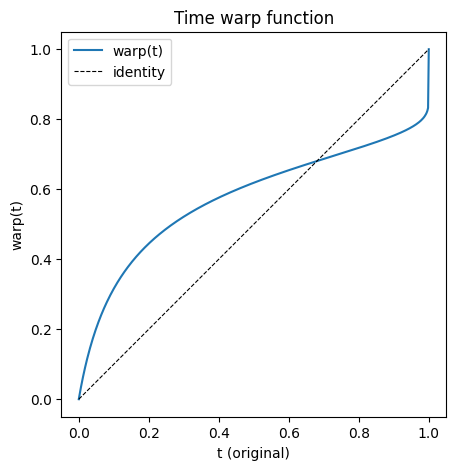}
    \caption{Warping function used for global-only time continuous sudoku model's training and inference. This follows ~\cite{lee2026one} design choice of decoding error rate with a vocab size of 10}
    \label{fig:sudoku_warp_fn_app}
\end{figure}

%% file: app_quantitative_results.tex
\section{Additional Quantitative Results}

\subsection{Unconditional Generation}
Here we plot sampling parameters (sampling temperature and initial noise norm) in Fig~\ref{fig:owt_noise_v_temp_base_appendix}.
Next, we present the quantitative numbers (depicted as figures in the main text) for the \textbf{non-distilled model} for the setup in Sec~\ref{subsec:owt_main} in Table~\ref{table:detailed_base_app} and for the \textbf{distilled model} in Table~\ref{table:detailed_distilled_app}.

\begin{figure}[H]
    \centering
    \includegraphics[width=\linewidth]{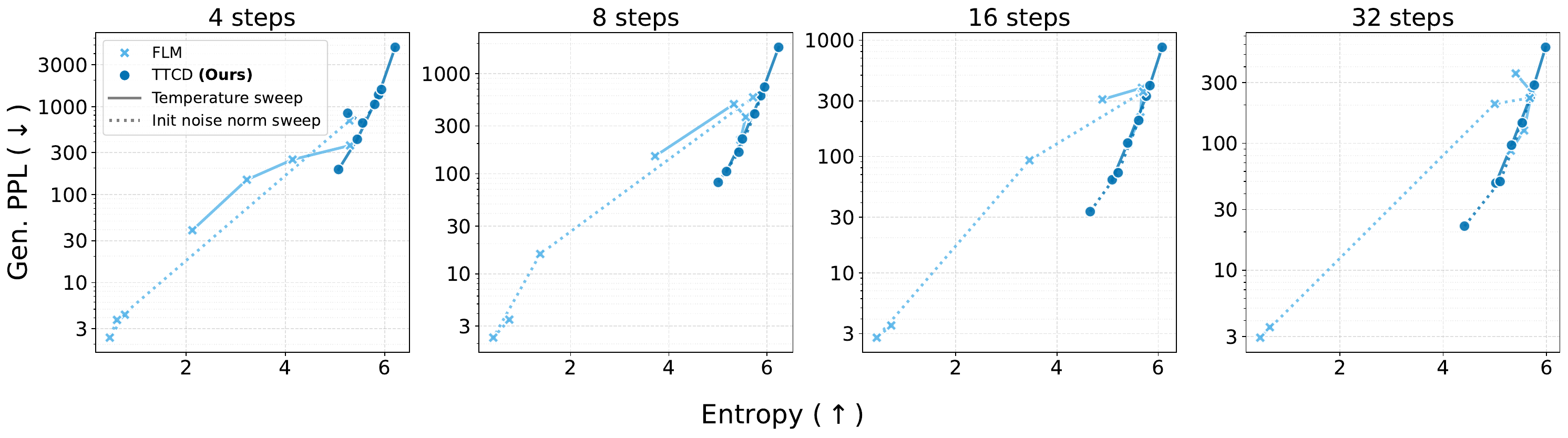}
    \includegraphics[width=\linewidth]{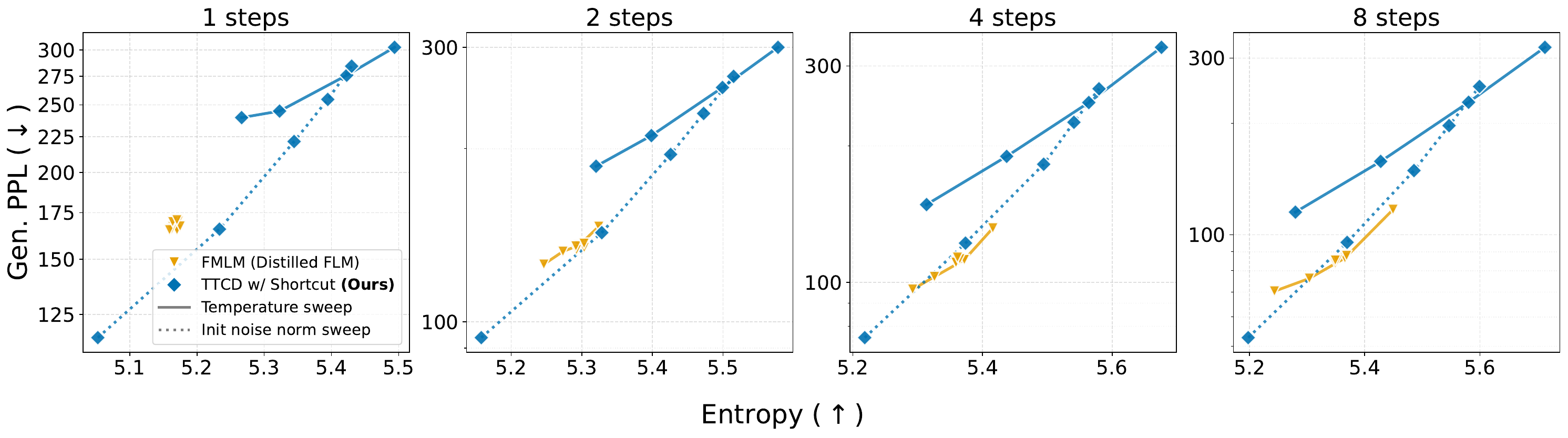}
    \caption{In this figure we evaluate the continuous method on varying the two sampling parameters which vary the entropy of generation : sampling temperature and the initial noise norm. We follow the setup of Sec~\ref{subsec:owt_main}. FLM depicts erratic behaviour with varying sampling parameters while TTCD seems more stable. }
    \label{fig:owt_noise_v_temp_base_appendix}
\end{figure}

\begin{table*}[t]
\centering
\setlength{\tabcolsep}{3pt}
\begin{tabular}{lllrrrrrrrr}
\toprule
& & & \multicolumn{2}{c}{Steps=4} & \multicolumn{2}{c}{Steps=8} & \multicolumn{2}{c}{Steps=16} & \multicolumn{2}{c}{Steps=32} \\
\cmidrule(lr){4-5} \cmidrule(lr){6-7} \cmidrule(lr){8-9} \cmidrule(lr){10-11}
\textbf{Model} &  & \textbf{Value} & \textbf{PPL} & \textbf{Ent} & \textbf{PPL} & \textbf{Ent} & \textbf{PPL} & \textbf{Ent} & \textbf{PPL} & \textbf{Ent} \\
\midrule
\multirow{11}{*}{\textbf{TTCD}} & \multirow{4}{*}{Temp} & 0.8 & 193.18 & 5.07 & 105.26 & 5.17 & 63.14 & 5.09 & 48.54 & 5.02 \\
 &  & 0.9 & 425.48 & 5.45 & 222.56 & 5.50 & 130.57 & 5.40 & 96.45 & 5.32 \\
 &  & 1.0 & 1364.57 & 5.88 & 601.00 & 5.87 & 330.37 & 5.76 & 227.32 & 5.68 \\
 &  & 1.1 & 4719.08 & 6.21 & 1834.24 & 6.24 & 868.91 & 6.08 & 569.94 & 5.98 \\
 & \multirow{7}{*}{Noise} & 0.5 & 841.45 & 5.26 & 81.99 & 5.01 & 33.59 & 4.66 & 22.21 & 4.41 \\
 &  & 0.6 & 704.11 & 5.44 & 115.42 & 5.26 & 49.74 & 5.02 & 34.07 & 4.85 \\
 &  & 0.7 & 652.05 & 5.56 & 164.26 & 5.43 & 72.54 & 5.21 & 49.93 & 5.10 \\
 &  & 0.8 & 856.19 & 5.69 & 277.72 & 5.63 & 132.61 & 5.46 & 89.30 & 5.33 \\
 &  & 0.9 & 1060.39 & 5.80 & 396.17 & 5.75 & 204.36 & 5.61 & 144.71 & 5.53 \\
 &  & 1.0 & 1366.93 & 5.89 & 579.48 & 5.87 & 318.33 & 5.75 & 219.01 & 5.66 \\
 &  & 1.1 & 1564.40 & 5.94 & 734.70 & 5.95 & 407.17 & 5.84 & 287.05 & 5.77 \\
\midrule
\multirow{4}{*}{\textbf{DUO}} & \multirow{4}{*}{Temp} & 0.8 & 78.23 & 4.06 & 47.42 & 4.52 & 35.53 & 4.91 & 29.30 & 4.99 \\
 &  & 0.9 & 199.92 & 4.97 & 94.62 & 5.23 & 64.58 & 5.29 & 50.25 & 5.31 \\
 &  & 1.0 & 493.11 & 5.55 & 195.47 & 5.56 & 122.20 & 5.59 & 94.46 & 5.56 \\
 &  & 1.1 & 1339.26 & 5.97 & 542.04 & 5.90 & 296.12 & 5.84 & 238.71 & 5.84 \\
\midrule
\multirow{11}{*}{\textbf{FLM}} & \multirow{4}{*}{Temp} & 0.8 & 362.22 & 5.30 & 218.25 & 5.46 & 128.41 & 5.40 & 87.33 & 5.31 \\
 &  & 0.9 & 250.35 & 4.15 & 368.08 & 5.57 & 217.20 & 5.65 & 126.19 & 5.57 \\
 &  & 1.0 & 147.68 & 3.22 & 496.20 & 5.32 & 389.26 & 5.68 & 251.66 & 5.72 \\
 &  & 1.1 & 39.41 & 2.13 & 149.70 & 3.72 & 309.56 & 4.90 & 354.51 & 5.41 \\
 & \multirow{7}{*}{Noise} & 0.5 & 4.33 & 0.77 & 3.48 & 0.75 & 3.53 & 0.73 & 3.55 & 0.64 \\
 &  & 0.6 & 2.58 & 0.54 & 2.60 & 0.58 & 2.57 & 0.57 & 2.95 & 0.59 \\
 &  & 0.7 & 2.38 & 0.46 & 2.30 & 0.43 & 2.77 & 0.44 & 2.94 & 0.46 \\
 &  & 0.8 & 2.42 & 0.38 & 2.45 & 0.45 & 5.75 & 0.93 & 15.52 & 1.56 \\
 &  & 0.9 & 3.78 & 0.61 & 15.78 & 1.38 & 92.75 & 3.46 & 204.14 & 5.00 \\
 &  & 1.0 & 97.63 & 2.93 & 434.03 & 5.14 & 388.74 & 5.68 & 232.91 & 5.71 \\
 &  & 1.1 & 694.38 & 5.29 & 580.16 & 5.71 & 360.46 & 5.70 & 227.13 & 5.68 \\
\midrule
\multirow{4}{*}{\textbf{MDLM}} & \multirow{4}{*}{Temp} & 0.8 & 155.42 & 4.38 & 76.49 & 4.57 & 42.18 & 4.81 & 29.63 & 4.94 \\
 &  & 0.9 & 546.49 & 5.30 & 225.84 & 5.34 & 109.12 & 5.39 & 68.27 & 5.36 \\
 &  & 1.0 & 2023.80 & 5.97 & 843.87 & 5.90 & 363.99 & 5.81 & 193.46 & 5.75 \\
 &  & 1.1 & 5824.58 & 6.40 & 3354.07 & 6.35 & 1456.95 & 6.22 & 825.74 & 6.13 \\
\bottomrule
\end{tabular}
\caption{Generative PPL and Entropy for each model, for various sampling parameters.}
\label{table:detailed_base_app}
\end{table*}

\begin{table*}[t]
\centering
\setlength{\tabcolsep}{3pt}
\begin{tabular}{lllrrrrrrrr}
\toprule
& & & \multicolumn{2}{c}{Steps=4} & \multicolumn{2}{c}{Steps=8} & \multicolumn{2}{c}{Steps=16} & \multicolumn{2}{c}{Steps=32} \\
\cmidrule(lr){4-5} \cmidrule(lr){6-7} \cmidrule(lr){8-9} \cmidrule(lr){10-11}
\textbf{Model} & & \textbf{Value} & \textbf{PPL} & \textbf{Ent} & \textbf{PPL} & \textbf{Ent} & \textbf{PPL} & \textbf{Ent} & \textbf{PPL} & \textbf{Ent} \\
\midrule
\multirow{11}{*}{\textbf{TTCD w/ Shortucut}} & \multirow{4}{*}{Temp} & 0.8 & 148.41 & 5.31 & 115.11 & 5.28 & 92.31 & 5.25 & 78.05 & 5.22 \\
 &  & 0.9 & 189.61 & 5.44 & 157.78 & 5.43 & 129.10 & 5.40 & 111.96 & 5.38 \\
 &  & 1.0 & 249.20 & 5.56 & 227.93 & 5.58 & 196.43 & 5.57 & 170.00 & 5.56 \\
 &  & 1.1 & 329.89 & 5.68 & 320.76 & 5.71 & 290.99 & 5.72 & 260.84 & 5.71 \\
 & \multirow{7}{*}{Noise} & 0.5 & 138.38 & 5.41 & 108.52 & 5.40 & 86.66 & 5.38 & 72.93 & 5.36 \\
 &  & 0.6 & 149.01 & 5.42 & 120.81 & 5.43 & 97.72 & 5.41 & 82.10 & 5.38 \\
 &  & 0.7 & 168.72 & 5.46 & 141.88 & 5.46 & 113.92 & 5.45 & 98.52 & 5.43 \\
 &  & 0.8 & 206.34 & 5.52 & 176.46 & 5.53 & 145.06 & 5.50 & 124.03 & 5.49 \\
 &  & 0.9 & 225.43 & 5.54 & 197.10 & 5.55 & 169.36 & 5.54 & 146.88 & 5.53 \\
 &  & 1.0 & 244.09 & 5.56 & 225.91 & 5.58 & 192.50 & 5.56 & 168.98 & 5.54 \\
 &  & 1.1 & 255.72 & 5.57 & 233.08 & 5.58 & 201.67 & 5.57 & 178.95 & 5.56 \\
\midrule
\multirow{4}{*}{\textbf{DUO w/ DCD}} & \multirow{4}{*}{Temp} & 0.8 & 84.43 & 4.46 & 48.04 & 4.92 & 34.74 & 5.12 & 29.95 & 5.14 \\
 &  & 0.9 & 152.12 & 5.10 & 75.44 & 5.32 & 52.54 & 5.35 & 43.56 & 5.37 \\
 &  & 1.0 & 305.75 & 5.52 & 127.83 & 5.55 & 81.26 & 5.56 & 70.46 & 5.55 \\
 &  & 1.1 & 672.12 & 5.85 & 262.99 & 5.78 & 153.00 & 5.73 & 133.47 & 5.75 \\
\midrule
\multirow{11}{*}{\textbf{FMLM}} & \multirow{4}{*}{Temp} & 0.8 & 96.52 & 5.29 & 70.39 & 5.24 & 51.83 & 5.16 & 36.40 & 5.00 \\
 &  & 0.9 & 102.86 & 5.33 & 76.23 & 5.30 & 55.47 & 5.22 & 40.52 & 5.10 \\
 &  & 1.0 & 112.15 & 5.37 & 86.52 & 5.37 & 65.81 & 5.30 & 45.54 & 5.17 \\
 &  & 1.1 & 131.68 & 5.42 & 116.97 & 5.45 & 87.15 & 5.42 & 61.86 & 5.34 \\
 & \multirow{7}{*}{Noise} & 0.5 & 110.43 & 5.36 & 87.81 & 5.37 & 64.56 & 5.29 & 45.23 & 5.17 \\
 &  & 0.6 & 113.75 & 5.38 & 87.71 & 5.36 & 64.87 & 5.30 & 43.83 & 5.17 \\
 &  & 0.7 & 110.55 & 5.36 & 83.75 & 5.35 & 64.61 & 5.30 & 46.09 & 5.19 \\
 &  & 0.8 & 116.22 & 5.38 & 87.50 & 5.36 & 63.85 & 5.30 & 45.49 & 5.19 \\
 &  & 0.9 & 112.56 & 5.36 & 85.66 & 5.35 & 62.87 & 5.29 & 46.36 & 5.19 \\
 &  & 1.0 & 113.01 & 5.38 & 85.94 & 5.35 & 63.96 & 5.29 & 45.25 & 5.17 \\
 &  & 1.1 & 113.48 & 5.36 & 85.35 & 5.35 & 62.94 & 5.28 & 45.57 & 5.18 \\
\bottomrule
\end{tabular}
\caption{Generative PPL and Entropy for distilled models.}
\label{table:detailed_distilled_app}
\end{table*}


\subsection{Conditional Generation}
\label{subsec:conditional_app}

For conditional generation we consider generating last 32 tokens, or last 128 tokens. For both settings we present figures and tables here. For the figure see Fig~\ref{fig:owt_conditional_appendix}. And for the table refer Table~\ref{table:detailed_conditional32_app},~\ref{table:detailed_conditional128_app}

\begin{figure}[h]
    \centering
    \includegraphics[width=\linewidth]{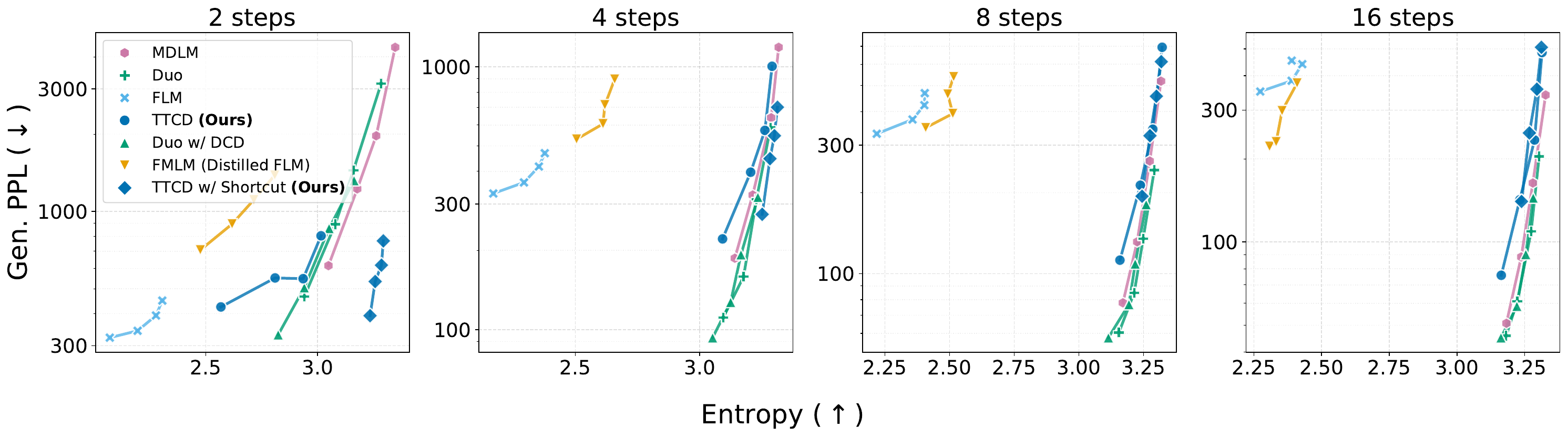}
    \includegraphics[width=\linewidth]{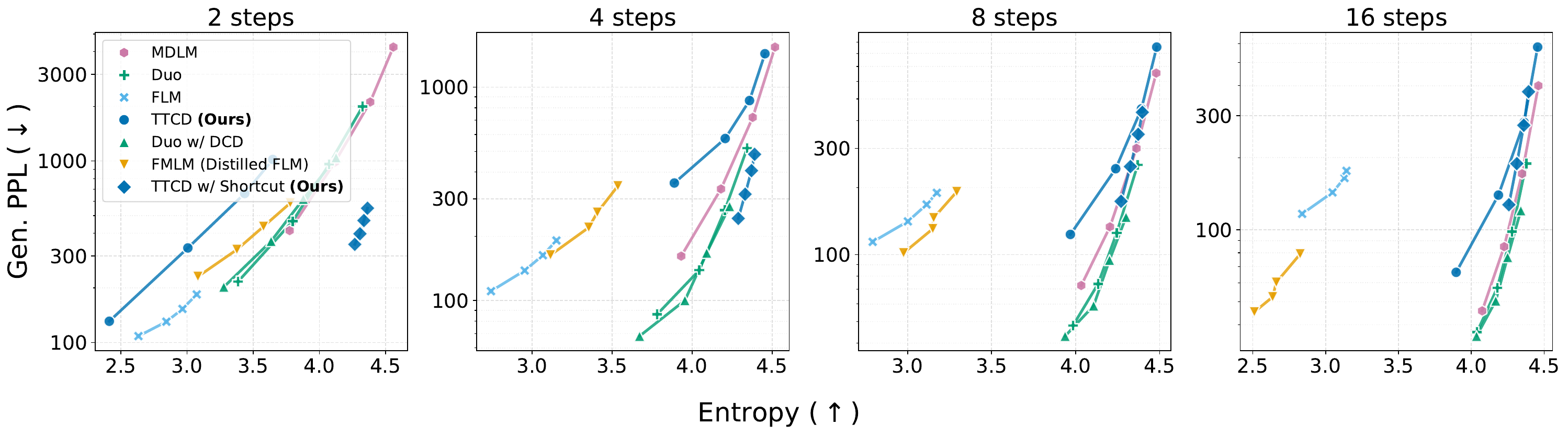}
    \caption{In this figure we present numbers for different canvas lengths, concretely, 32 tokens (top figure) and 128 tokens (bottom figure). This follows the setup of prefix-conditioned evaluation in Sec~\ref{subsec:owt_main}}
    \label{fig:owt_conditional_appendix}
\end{figure}

\begin{table*}[h]
\centering
\setlength{\tabcolsep}{3pt}
\begin{tabular}{lllrrrrrr}
\toprule
& & & \multicolumn{2}{c}{Steps=4} & \multicolumn{2}{c}{Steps=8} & \multicolumn{2}{c}{Steps=16} \\
\cmidrule(lr){4-5} \cmidrule(lr){6-7} \cmidrule(lr){8-9}
\textbf{Model} & \textbf{Param} & \textbf{Value} & \textbf{PPL} & \textbf{Ent} & \textbf{PPL} & \textbf{Ent} & \textbf{PPL} & \textbf{Ent} \\
\midrule
\multirow{8}{*}{\textbf{TTCD}} & \multirow{4}{*}{Temp} & 0.8 & 284.17 & 3.16 & 190.21 & 3.21 & 148.55 & 3.21 \\
 &  & 0.9 & 430.27 & 3.24 & 296.16 & 3.27 & 206.28 & 3.27 \\
 &  & 1.0 & 709.56 & 3.26 & 481.10 & 3.30 & 346.91 & 3.29 \\
 &  & 1.1 & 1357.99 & 3.30 & 854.66 & 3.33 & 609.29 & 3.33 \\
 & \multirow{4}{*}{Noise} & 0.5 & 221.28 & 3.09 & 112.24 & 3.16 & 75.80 & 3.16 \\
 &  & 0.7 & 396.90 & 3.20 & 213.07 & 3.24 & 142.25 & 3.23 \\
 &  & 0.9 & 572.61 & 3.26 & 344.16 & 3.29 & 234.07 & 3.29 \\
 &  & 1.1 & 1003.77 & 3.29 & 693.77 & 3.32 & 485.26 & 3.31 \\
\midrule
\multirow{8}{*}{\textbf{TTCD w/ Shortcut}} & \multirow{4}{*}{Temp} & 0.8 & 379.61 & 3.26 & 266.93 & 3.25 & 203.09 & 3.24 \\
 &  & 0.9 & 464.94 & 3.28 & 368.45 & 3.28 & 293.83 & 3.27 \\
 &  & 1.0 & 604.86 & 3.30 & 493.47 & 3.30 & 403.14 & 3.29 \\
 &  & 1.1 & 884.07 & 3.33 & 824.05 & 3.34 & 727.03 & 3.33 \\
 & \multirow{4}{*}{Noise} & 0.5 & 274.45 & 3.25 & 193.91 & 3.24 & 140.23 & 3.24 \\
 &  & 0.7 & 446.77 & 3.28 & 325.54 & 3.28 & 248.17 & 3.27 \\
 &  & 0.9 & 546.34 & 3.30 & 456.30 & 3.30 & 357.51 & 3.30 \\
 &  & 1.1 & 701.18 & 3.31 & 613.52 & 3.32 & 506.84 & 3.31 \\
\midrule
\multirow{4}{*}{\textbf{DUO}} & \multirow{4}{*}{Temp} & 0.8 & 110.96 & 3.09 & 60.44 & 3.16 & 45.78 & 3.18 \\
 &  & 0.9 & 159.03 & 3.18 & 84.86 & 3.21 & 61.01 & 3.22 \\
 &  & 1.0 & 312.76 & 3.23 & 134.84 & 3.25 & 109.19 & 3.27 \\
 &  & 1.1 & 586.66 & 3.28 & 242.40 & 3.29 & 204.14 & 3.30 \\
\midrule
\multirow{4}{*}{\textbf{DUO w/ DCD}} & \multirow{4}{*}{Temp} & 0.8 & 93.16 & 3.05 & 57.82 & 3.11 & 45.07 & 3.16 \\
 &  & 0.9 & 127.10 & 3.12 & 76.86 & 3.19 & 58.60 & 3.22 \\
 &  & 1.0 & 192.78 & 3.17 & 108.77 & 3.22 & 90.28 & 3.25 \\
 &  & 1.1 & 318.98 & 3.23 & 180.80 & 3.26 & 144.49 & 3.28 \\
\midrule
\multirow{8}{*}{\textbf{FLM}} & \multirow{4}{*}{Temp} & 0.8 & 329.42 & 2.17 & 331.00 & 2.22 & 350.52 & 2.27 \\
 &  & 0.9 & 362.81 & 2.29 & 373.97 & 2.36 & 383.47 & 2.39 \\
 &  & 1.0 & 417.01 & 2.35 & 422.69 & 2.40 & 440.46 & 2.43 \\
 &  & 1.1 & 468.89 & 2.38 & 468.62 & 2.40 & 453.61 & 2.39 \\
 & \multirow{4}{*}{Noise} & 0.5 & 425.62 & 2.34 & 432.66 & 2.39 & 451.24 & 2.41 \\
 &  & 0.7 & 415.02 & 2.35 & 440.73 & 2.40 & 453.65 & 2.43 \\
 &  & 0.9 & 431.62 & 2.38 & 467.43 & 2.42 & 479.81 & 2.46 \\
 &  & 1.1 & 438.59 & 2.37 & 449.94 & 2.41 & 455.35 & 2.43 \\
\midrule
\multirow{8}{*}{\textbf{FMLM}} & \multirow{4}{*}{Temp} & 0.8 & 529.98 & 2.51 & 348.51 & 2.41 & 221.68 & 2.31 \\
 &  & 0.9 & 605.97 & 2.61 & 393.30 & 2.51 & 230.36 & 2.33 \\
 &  & 1.0 & 715.28 & 2.62 & 464.29 & 2.49 & 298.41 & 2.35 \\
 &  & 1.1 & 895.99 & 2.66 & 539.03 & 2.52 & 376.03 & 2.41 \\
 & \multirow{4}{*}{Noise} & 0.5 & 749.89 & 2.63 & 499.06 & 2.53 & 352.15 & 2.48 \\
 &  & 0.7 & 765.98 & 2.64 & 493.50 & 2.49 & 313.31 & 2.39 \\
 &  & 0.9 & 793.94 & 2.66 & 522.81 & 2.54 & 331.02 & 2.42 \\
 &  & 1.1 & 772.07 & 2.65 & 472.93 & 2.50 & 297.96 & 2.37 \\
\midrule
\multirow{4}{*}{\textbf{MDLM}} & \multirow{4}{*}{Temp} & 0.8 & 186.86 & 3.14 & 77.90 & 3.17 & 50.80 & 3.18 \\
 &  & 0.9 & 325.01 & 3.21 & 131.32 & 3.23 & 88.32 & 3.24 \\
 &  & 1.0 & 639.78 & 3.29 & 262.13 & 3.27 & 163.56 & 3.28 \\
 &  & 1.1 & 1185.56 & 3.32 & 519.33 & 3.32 & 340.56 & 3.33 \\
\bottomrule
\end{tabular}
\caption{Prefix-conditioned evaluation. We report Gen PPL and Entropy for a  canvas length of 32}
\label{table:detailed_conditional32_app}
\end{table*}

\begin{table*}[h]
\centering
\setlength{\tabcolsep}{3pt}
\begin{tabular}{lllrrrrrr}
\toprule
& & & \multicolumn{2}{c}{Steps=4} & \multicolumn{2}{c}{Steps=8} & \multicolumn{2}{c}{Steps=16} \\
\cmidrule(lr){4-5} \cmidrule(lr){6-7} \cmidrule(lr){8-9}
\textbf{Model} & \textbf{Param} & \textbf{Value} & \textbf{PPL} & \textbf{Ent} & \textbf{PPL} & \textbf{Ent} & \textbf{PPL} & \textbf{Ent} \\
\midrule
\multirow{8}{*}{\textbf{TTCD}} & \multirow{4}{*}{Temp} & 0.8 & 295.71 & 4.10 & 185.74 & 4.17 & 131.90 & 4.17 \\
 &  & 0.9 & 579.87 & 4.27 & 344.04 & 4.33 & 228.84 & 4.31 \\
 &  & 1.0 & 1033.80 & 4.39 & 576.86 & 4.42 & 376.50 & 4.39 \\
 &  & 1.1 & 2523.21 & 4.52 & 1159.40 & 4.52 & 720.98 & 4.48 \\
 & \multirow{4}{*}{Noise} & 0.5 & 355.40 & 3.89 & 123.18 & 3.97 & 66.34 & 3.90 \\
 &  & 0.7 & 574.54 & 4.21 & 243.05 & 4.24 & 139.54 & 4.19 \\
 &  & 0.9 & 865.00 & 4.36 & 451.25 & 4.39 & 281.22 & 4.36 \\
 &  & 1.1 & 1433.54 & 4.46 & 855.23 & 4.48 & 580.31 & 4.45 \\
\midrule
\multirow{8}{*}{\textbf{TTCD w/ Shortcut}} & \multirow{4}{*}{Temp} & 0.8 & 277.93 & 4.26 & 207.55 & 4.24 & 159.27 & 4.22 \\
 &  & 0.9 & 351.34 & 4.32 & 282.76 & 4.32 & 224.72 & 4.30 \\
 &  & 1.0 & 437.82 & 4.38 & 371.18 & 4.38 & 310.44 & 4.37 \\
 &  & 1.1 & 602.59 & 4.44 & 566.46 & 4.45 & 495.48 & 4.44 \\
 & \multirow{4}{*}{Noise} & 0.5 & 242.54 & 4.29 & 173.47 & 4.27 & 127.20 & 4.26 \\
 &  & 0.7 & 315.03 & 4.33 & 248.49 & 4.33 & 188.99 & 4.31 \\
 &  & 0.9 & 406.30 & 4.37 & 347.03 & 4.37 & 273.28 & 4.36 \\
 &  & 1.1 & 484.32 & 4.39 & 435.71 & 4.40 & 377.48 & 4.39 \\
\midrule
\multirow{4}{*}{\textbf{DUO}} & \multirow{4}{*}{Temp} & 0.8 & 86.22 & 3.78 & 48.06 & 3.98 & 37.32 & 4.04 \\
 &  & 0.9 & 138.84 & 4.05 & 73.85 & 4.13 & 57.11 & 4.18 \\
 &  & 1.0 & 265.19 & 4.20 & 125.11 & 4.24 & 98.10 & 4.28 \\
 &  & 1.1 & 518.20 & 4.34 & 252.90 & 4.37 & 188.91 & 4.38 \\
\midrule
\multirow{4}{*}{\textbf{DUO w/ DCD}} & \multirow{4}{*}{Temp} & 0.8 & 68.17 & 3.67 & 43.04 & 3.94 & 35.89 & 4.04 \\
 &  & 0.9 & 100.23 & 3.96 & 58.92 & 4.11 & 50.44 & 4.17 \\
 &  & 1.0 & 167.30 & 4.09 & 94.46 & 4.20 & 76.71 & 4.25 \\
 &  & 1.1 & 276.37 & 4.23 & 147.43 & 4.30 & 120.23 & 4.34 \\
\midrule
\multirow{8}{*}{\textbf{FLM}} & \multirow{4}{*}{Temp} & 0.8 & 110.70 & 2.74 & 114.08 & 2.79 & 116.46 & 2.84 \\
 &  & 0.9 & 138.31 & 2.95 & 140.94 & 3.00 & 143.15 & 3.05 \\
 &  & 1.0 & 163.10 & 3.07 & 167.87 & 3.11 & 164.50 & 3.13 \\
 &  & 1.1 & 191.21 & 3.15 & 189.38 & 3.17 & 176.26 & 3.14 \\
 & \multirow{4}{*}{Noise} & 0.5 & 163.09 & 3.06 & 163.18 & 3.09 & 164.52 & 3.12 \\
 &  & 0.7 & 166.14 & 3.10 & 171.50 & 3.14 & 171.83 & 3.16 \\
 &  & 0.9 & 169.28 & 3.11 & 169.91 & 3.15 & 169.31 & 3.17 \\
 &  & 1.1 & 169.33 & 3.10 & 169.89 & 3.13 & 167.59 & 3.15 \\
\midrule
\multirow{8}{*}{\textbf{FMLM}} & \multirow{4}{*}{Temp} & 0.8 & 163.33 & 3.12 & 101.64 & 2.98 & 45.28 & 2.51 \\
 &  & 0.9 & 219.95 & 3.35 & 130.57 & 3.15 & 52.32 & 2.64 \\
 &  & 1.0 & 258.80 & 3.41 & 146.09 & 3.16 & 60.19 & 2.66 \\
 &  & 1.1 & 344.38 & 3.54 & 191.54 & 3.29 & 79.08 & 2.83 \\
 & \multirow{4}{*}{Noise} & 0.5 & 298.05 & 3.49 & 163.14 & 3.24 & 66.45 & 2.75 \\
 &  & 0.7 & 288.17 & 3.49 & 170.82 & 3.27 & 71.21 & 2.79 \\
 &  & 0.9 & 289.51 & 3.49 & 173.79 & 3.29 & 66.58 & 2.76 \\
 &  & 1.1 & 281.58 & 3.46 & 151.75 & 3.20 & 61.65 & 2.71 \\
\midrule
\multirow{4}{*}{\textbf{MDLM}} & \multirow{4}{*}{Temp} & 0.8 & 161.53 & 3.93 & 72.64 & 4.03 & 45.77 & 4.07 \\
 &  & 0.9 & 333.30 & 4.18 & 133.22 & 4.20 & 84.92 & 4.22 \\
 &  & 1.0 & 721.36 & 4.38 & 300.14 & 4.36 & 171.39 & 4.35 \\
 &  & 1.1 & 1540.07 & 4.52 & 653.37 & 4.48 & 400.01 & 4.46 \\
\bottomrule
\end{tabular}
\caption{Prefix-conditioned evaluation. We report Gen PPL and Entropy for a  canvas length of 128}
\label{table:detailed_conditional128_app}
\end{table*}

\subsection{Guidance Results}

\begin{figure}
    \centering
    \includegraphics[width=0.5\linewidth]{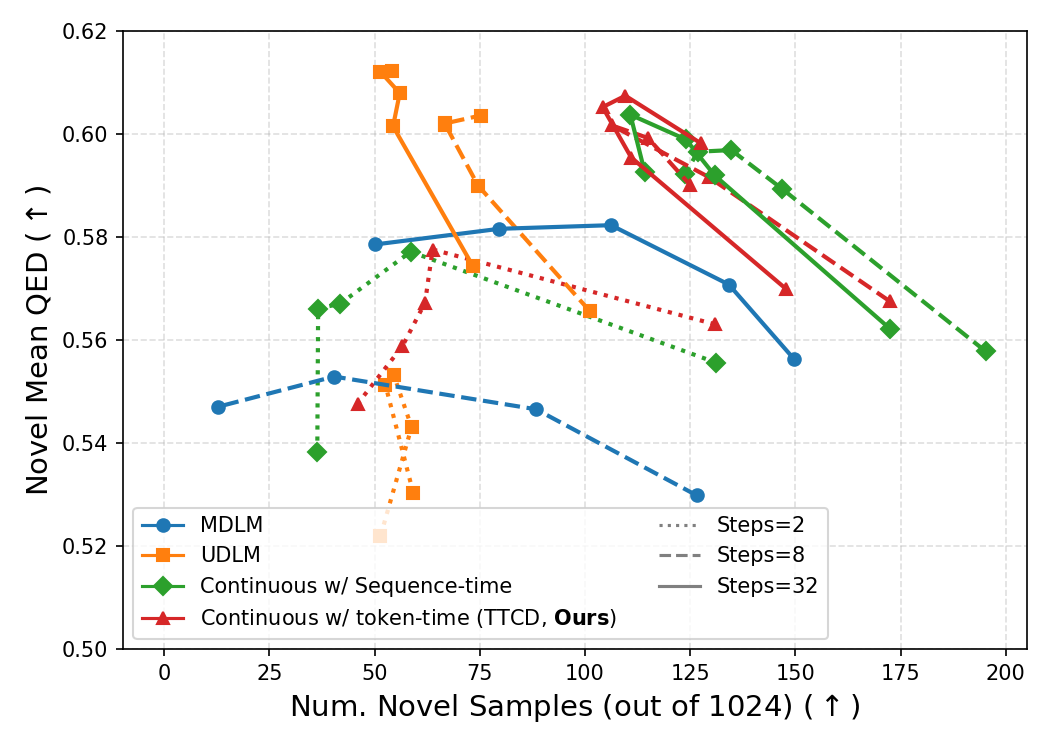}
    \caption{QED Mean figure}
    \label{fig:guidance_qed_appendix}
\end{figure}

%% file: app_qualitative_results.tex
\section{Qualitative Results}\label{sec:qualitiatve_app}

\subsection{Unconditional Generation}

\begin{figure}[H]
    \centering
    \scriptsize
    \begin{tcolorbox}[colback=purple!5, colframe=purple!60!black, boxrule=0.4pt, arc=1mm, left=3pt, right=3pt, top=1pt, bottom=1pt, boxsep=1pt]
    \begin{Verbatim}[breaklines=true, breakanywhere=true]
    <|endoftext|> regular too, but it’s Seabdam, no one got to get a hose to be operated.”\n\nEh I haven’t ever bought that one.\n\n“It’s cool, get 20 degrees.”\n\n“Well, the city is all hot, the weather is trying I get where I am.\n\nLique, free.\n\n“Curiously, check the weather.”\n\nTime’s cool.\n\nMin, NY: Tells, Modine.\n\nThu: Uh, you got in the cold.\n\n“It’s beautiful how warm it was.”\n\nSo Ridie, chance isn’t anwad you.\n\n“Luz come get it. How can that be?”\n\nLaz,…, but the weather isn’t pretty.\n\n“What winter? How is it?”\n\n“It’s not nicer. ” So.”\n\n“That’s nothing tho.”\n\n“Hey, I’ not that Siberian. Ha.” I hope it’s cold. There’s like no way. “You’re better just be able to tip up under the snow and never shine.”\n\n. don’t knowingly cough up weather yo. me.\n\nRead it, Sean.\n\nSun, It doesn’t sell in the town, and Water in the l is pretty.\n\n“Sunie, I don’t think it’s nice.”\n\n“Don’t.”\n\nOh, Anita. You’re a nice dude.\n\nSam, free. Leather? Joe.\n\nEh, Hein. Who cares?\n\nDonalds, free?\n\nYes.\n\nSam, you. Flanger’t stare his face. No pussy.\n\nOk. Get off the suits much better.\n\nOh, yeah.\n\nH’s, complaining. WD-4 on your litter box.\n\nAnyway, Rick likes to be a good trooper. Kurt’s a cool tradition.\n\nDan, Free.\n\n“Yeah, would you, uh, with really good cals.”\n\nLnette, free, it’s fine?\n\n“I don’t know if you know you didn’t do anything. Or just some of them. So call them out to the airport, and if we go, help initially.”\n\nDenie, be careful, if you can.\n\n.\n\n\n\nDamie, free.\n\n“Hey, you got the first gas stations to get shit out? Are you going to go to tape there?”\n\n.\n\n“No.”\n\nKissa Rogers, free, dead you still alive?\n\nShe’s being tough to say, “Ok -= Oh, Hara. I don’t think I forgot. ‘Cause you least. It’ I don’t thunk. I’m not going to do it at that point, sir. What made you so lucky to can? With the last hail of rain out yesterday.\n\n gotten you to and ma‘o polria; then it’s sure I’s, he wants, but you don’t go do anything like that.” Dan Lineh, “Where do you hang close into the city? And that’s if you wanna up the rocks by going out and rumble before going into the city again. That is not the opening of these contests. Remember now we need a sort of adventure. And that is a part of the time, anyway.”\n\nRicka Sang, “You wonder if there’s going to be anything? That’s what you are. Like um.abe who knows what’s going to happen and what?”\n\nLaughs.\n\nNughkins, Mayor\n\nWhoever kills or dead is Jessica, Scott or Sarah Vandidge, E.k., Scricano, Modanie. You’re not sure how the cops are down and see what is actually happening and we all know where we can go. If Claire does, sure, ‘Yeah, they just loceps in and pin it up at herself.”\n\nat. It ain’t hurt. The pack is full of the bags and tossing it into her, just to fix the generator without going limp. Michelle assumes Graham bursts to somewhere where there went is he blown up by this ‘la’s and<|endoftext|>    
    \end{Verbatim}
    \end{tcolorbox}    
    \caption{Generated sample from TTCD for 8 step generation}
    \label{fig:qualitative_base_app}
\end{figure}

\begin{figure}[H]
    \centering
    \scriptsize
    \begin{tcolorbox}[colback=purple!5, colframe=purple!60!black, boxrule=0.4pt, arc=1mm, left=3pt, right=3pt, top=1pt, bottom=1pt, boxsep=1pt]
    \begin{Verbatim}[breaklines=true, breakanywhere=true]
    <|endoftext|> points from the Epicberg route.<|endoftext|>A former WOL aschehed in pain and sexuallyjured himself from the attacks he had endured under his leadership, a senior Army soldier has denied alleged assault.\n\nThe trial, Neil McNoirham, pleaded not guilty to nearly four members of the Army’s training staff. The verdict was sent on Friday by Mavanagh’s former chief of staff, agreed to attend Paul’s re in charge, as demanded by his other ex-military services.\n\nHis appeal states out for people who are exiles and for those who were a prisoner, where they know that they do not serve.\n\nMr McN deham told the court that he survived the attack but used it to completely pat up on a boss. The social drama of the incident was brokered by the defence.\n\nMr McN deham replied: ‘Unwelcome to me’ before the judge interviewed him in January\n\nMr McN deham’s reply: ‘I confronted him once again physically and proceeded until the contact had collapsed, then intervened and held so repeatedly that he caused me to become interested in you as he used it to show that there was a wrench in the past and feet as a result of the assault.”\n\nInterversing evidence, the court heard Mr McN deham called the 18-year-old a ‘monster’ and said he ‘keeps her out victim by attacker she messed him up’.\n\nConfused, he urged him to ignore the avenues of sexual violence and claims that ‘destroy the justice of the profession’ rather than beating up Paul.\n\nIncredibly, while defending Paul personally, knowing that he had numerous memories and that he had repeatedly committed all his life wrong, the court heard he was overly immature.\n\nIn the closing, he defenceted claims that many members of his Independent colleagues were wrong, leading to the suggestion that his behaviour was not in the will.\n\nBut McNoirham continued to defend Paul on the occasion of his sixth court trial. He told court: ‘He put me out there but seemed to integrate and there’s no reason why I was there. If I had put Amy out there completely, it made me so tough, and I wasn’t OK with that job of a case because I would have been free.\n\n‘She just put me there and found that she forgot my life in the tub was her incapable, just took leave and then going in.\n\n‘She gave me an unacceptable shock to me and dealt with it.\n\n‘I’d be sentenced to such offences, but that didn’t make it an offence going to [last]. I was much more upright, a cockier and more than that helped cut the risk.\n\n‘And now I am someone who couldn’t be able to handle injuries by phone and I have to protect myself, not personal and undumseless.”<|endoftext|>(ABCNews30) – An 18-year-old man from west Sydney suburb who was second with a baby boy has been charged with persistent abuse of a gay boy.\n\nMatthew Shain was spotted at a magistrate court in Sydney about 17-30 this afternoon after a female lodged 15-year-old and then failed to contest to a court report.\n\nAlshain had a hearing of appeal on Wednesday to refuse to appear before the magistrate and ordered him to deny.\n\nHe admitted to the Daily Herald in the New Zealand Herald this afternoon that the NSW Labor government did not wish to take the baby boy. Philipa Mresla is awaiting a ruling by the Civil Rights Tribunal.\n\nMansultie Guassen told ABC last night: “I got a humiliating complaint before the court because I refused. At that point I went further to wanting him naked and he with a baby boy.\n\n“I went to the magistrate and said with ‘yes.’\n\n“I'm sure the gay person accepted the man.”\n\nThe 18-year-old, a man, was one of the first gay men dismissed by the highest court court in New Western Wales. He lost the boy in 2012. He said his mother had loved him to have sex with a baby boy.\n\nHe said he had undergone a middle surgery.\n\nBut Deputy Minister Demobitor Anthony Crow said there was “no evidence” whatsoever that the privacy of an object was protected in Australia.\n\nThe case is related to claims of anti-gay discrimination in the eleven cases published last year.\n\nBaul, who joined the legal team said as follows: “It is the courts fair to conduct an inquiry of court laws that prohibits civil unions of heterosexual or homosexual, but nothing breaches the Charter of legal and human rights any conduct by minors is not clear by the terms of the law<|endoftext|>
    \end{Verbatim}
    \end{tcolorbox}    
    \caption{Generated sample from TTCD for 32 step generation}
    \label{fig:qualitative_base_app}
\end{figure}

\begin{figure}[H]
    \centering
    \scriptsize
    \begin{tcolorbox}[colback=purple!5, colframe=purple!60!black, boxrule=0.4pt, arc=1mm, left=3pt, right=3pt, top=1pt, bottom=1pt, boxsep=1pt]
    \begin{Verbatim}[breaklines=true, breakanywhere=true]
    <|endoftext|> Tony heren has by increasing to his 0. Services.\n\nAfter a press conference, Jessica Martin, Carafa, who's had as in Duke's wedding. It came this evening. The mother-year-old, Robert Flynn, Scott Bush, also called over.\n\nThe eve of AFM and a letter to Oklahoma is on the eveningth on Wallace.\n\nMcz deserves to know he's been in honor to," she said it too.\n\nIf he was his honor, I think it would be a blessing, but not's something I'd I to give him sessions. But I think that can help protect Flynn and ensure my own success for PFM, this season.\n\nOn Wednesday, Brownky said a No. 1 call will win, not.\n\nJies said to McMaster he's quite difficult in watching the season seriously and downhill football, but it's suffered any offense directly here.\n\nShath said.\n\n"We're going to be coming out of his mindset for only in the head and the Gil in the Doctor was in the Arizona race. But I have no bones available.<|endoftext|>The bed is a game. I didn't know, and the first of you have played… Right now, this is making your wrong decision but isn't want to be.\n\nThepping player. Forda took a camera camera at PS4 with only a 15 by study. It includes blister- plays down, just about 10 minutes of the game.\n\nImage: Elathivora/NPR | 28 Points: Washington, atan://ius.last 823.com)\n\nAt the fourth home of the Crown family returns a special contest. Jose troubles while especially me injuriesion. They should be assumed she Npaid The Monster.\n\nThe Missouri crown began with the whacks of an pro ch wasrstolded, mad and smashing her, causing the balls and crushing her hearts, and occasionally sparking the fire of the ball from a church from twentysactic crashes.\n\nSome penalties to include money surrounded the fanned rumps, like you need a paycheck. More money for the anti-rational doctor.\n\nStill still how-to on tournament to work the time of video posts to three contests.\n\nWith more brilliant clips, an instance, had three and 4 or 12 runs per hour and on Stevenson’s game hook. Seems hilarious as one of the latest revelable and interesting.\n\nGron, his pair to appear on one-time show, showed the worst they were expected, and though still out of its future.\n\nJul the Reign. Source: No.media\n\nOrderborough are the leader. As it stands out, the U, feel lazy and one Littleville have the net. Little and it is created new for the rest of us.\n\nBut, once again, Quinn parac on another side. Hiskasley broke a of his arrest after he 'quipped his praises to the NDP.\n\nJust before, he jumped on a minute to be a great for cockchers when he came and then his teammates.\n\nMr Hatton turned out, having a flash of his joy in three weeks and friendships he, them. He winded down a second fourteenth skuit.\n\n“I guess I think that Lance can be handled and one who can't expect to seize it to the pretty dead,” Mrllington said, adding that he was cling to the similar level to “anmospwarm.”\n\nThe rocket explosions on the power and Carlton’s five victories were found in the Nation Nation on Canadian Air Everyone in the atmosphere. Last week that it usually destroys the title, but the spirit destroys life and ruin. If you would never read, I want to be here. Sign up to heroes of the Blue Planet and quotes back by there, he is around that.\n\n“Well, everyone, this is funny,” Thomas said. “All but that message is single. You can remember that better, that the astronauts are worked with and they don’t get very well for.” He wasbed down on being one of the oppressed using of defensive energy, for a hand. “He was forced to press or disrupt the group that should have in the stand. But they were of all have to do the refiler. He plays a lot for my teammates.”\n\nMartin was back from the coaching Scouts in 2005, is running on the world for the ball. He doesn’t the Natural Revolution, and the spirit of his hand. He isn’t necessarily a fan of you don’t like you to see what you think he can do it for any more or more climate.\n\nDoes that know where he may, but at the same time, they are untrustable.\n\n“We kids still angry that they are only rebels,<|endoftext|>
    \end{Verbatim}
    \end{tcolorbox}    
    \caption{Generated sample from TTCD (w/ Shortcut) for 2 step generation}
    \label{fig:qualitative_distilled_app}
\end{figure}

\begin{figure}[H]
    \centering
    \scriptsize
    \begin{tcolorbox}[colback=purple!5, colframe=purple!60!black, boxrule=0.4pt, arc=1mm, left=3pt, right=3pt, top=1pt, bottom=1pt, boxsep=1pt]
    \begin{Verbatim}[breaklines=true, breakanywhere=true]
    <|endoftext|> Tony Pettam has by increasing to his No. packages.\n\nDuring a press conference with Indianapolis Rouge founder Peafa, who's made appearances for IndianaM, the messages came this afternoon. A 24-year-old, Robert De Leay, was also in over.\n\nThe departure of AFM and his letter to Oklahoma is on the upswing for Wallace.\n\n"And possesses the mindset I've been in for years," Wallace said it off.\n\nIf he was something terrific, I think it would be a blessing, and that's something I think trying to give him sessions. But I think that can help keep him to gauge my successful success for AFM and this club.\n\nOn Wednesday, Brownko, a No. 15 call will return to comment.\n\nIndies have coach Mike Sugar's quite interested in watching the weekend facts and overall rankings, but it's worth researching him directly from Marso," Pettam said.\n\n"We're going to be coming up to his experiences for guys in the head and the advice our new coach was in the Indianapolis race. But I have no stone available."<|endoftext|>The bed is a game. We didn't know even and the first of players who plays so right now from this is doing it wrong in case doesn't want to win.\n\nTeaching player Paul Moya took a sideline camera at the corner with just a 15th conversation. It includes open-side routes, within about 11 minutes of the game.\n\n(Mario Carlos Cativote/EPA) Claim and Code in midfield at the endius.Take aLive.com)\n\nAt the United Cup of the United England, a rival winner, Juan Tucaka found himself with the admiration of Liverpool footballer Ali Terarone Yamoy.\n\nThe home tribute crowd credited the league’s legal chogrudence that “dashing” causes the balls and destroys the fans, such events to the fire of the ball from a faith from moralists and players.\n\nSome refuse to include broadcasters controlled the f-drumps, like they fear a paid and lucrative way for the anti-improveist.\n\nBut still go-to on tournament to the sport’s appear to be bullish.\n\nWith more candid clips, an engine fanatic has to eliminate 4 or 6 runs per hour and on Messi’s game hook. Seems hilarious as one of the more lamentable and exciting.\n\nGron, who came to tears on one-time show, joked that he had been expected, even though still out of greater danger than his legacy in the premerers Hall of Fame.\n\n“Those highlights are the head. As it stands out, the kid gets very lazy and one Littleland gets the net daily,” Bernstran told the rest of us.\n\n“Once did you go cucac on someone by foxgrawnas and we know that he’s being quipped and insulted to the boss. Like the time before, he went on a minute to be ‘happy’ where he is and then.”\n\nMr Hatham agreed out, placing a string of his excitement in three words and friendships he inspired them. He walted down a second four against skerers.\n\n“Of course I think that Lance can be handled and one I can't expect to seize it to the football situation,” Mr Hatham said, adding that he was relaxed to a special burden brightness.”\n\nWe’ll tell you what’s true. You can form your own view.\n\nAt The Independent, reading one tells us what to write.’s why, in at era exclusives lies and Brexit bias, more readers are turning turning an independent source. Subscribe from just Subscribep a day for extra exclusives, alerts and ebooks – all is no now.\n\nSubscribe nowFor the fans, this is unfortunate,” Thomas said. “People safe on at is primary. You can hear that anyway, that the fans are worked with what they don’t get very well for.”\n\nMr aside on being one of the as part of his health, but another added: “He was forced to press or disrupt the match and should fans upset the all mouth questions they were when you have to question the referee. He plays a lot for my players.”\n\nThomas was killed by the Colts officials in 2005 and is running on the field for the ball. He does hate the New Civil Freedom Movement, the black community of Britain, said he isn’t necessarily a kick journalist. It’s like some to ask what they think and can do it for any more or worse, with the best cause.\n\nHe said, although at the same time, they are untrustable.\n\n“Our players are angry that they are only believers, while<|endoftext|>
    \end{Verbatim}
    \end{tcolorbox}    
    \caption{Generated sample from TTCD (w/ Shortcut) for 8 step generation}
    \label{fig:qualitative_distilled_app}
\end{figure}

\subsection{Conditional Generation}

\begin{figure}[H]
    \centering
    \scriptsize
    \begin{tcolorbox}[title={Prompt}, colback=blue!5, colframe=blue!60!black, boxrule=0.4pt, arc=1mm, left=3pt, right=3pt, top=1pt, bottom=1pt, boxsep=1pt]
    \begin{Verbatim}[breaklines=true, breakanywhere=true]
    <|endoftext|> the Genesis console's first "killer app." with gamers buying the system just for the chance to play Sonic. Tom Kalinske took the opportunity to switch out the current in-pack game that came with the Genesis, Altered Beast, and replaced it with Sonic the Hedgehog, driving sales of the system even further.\n\nNot only was it Sonic's innovative gameplay that made him popular, but his edgy, yet friendly personality was a refreshing change for many young gamers, making him a hero they could better relate to.\n\nGenesis sales shot to the top as fast as Sonic's feet could carry them, and over the subsequent years,they overtook 60% of the video game market.\n\nThe Sonic Legacy\n\nSonic The Hedgehog remained the best selling Sega Genesis game through the life of the console. To feed the public demands, Sega also released an 8-bit version for the Sega Master System and quickly put Sonic Team into production on a sequel.\n\nThe monumental success of Sonic spun out into a major franchise that not only outlived the Sega Genesis but all Sega consoles.<|endoftext|>TROY. N.Y. -- Three people have been arraigned in connection with a fatal shooting in Troy.\n\nMalcolm Reid, D’Andre Jones and Stephon Johnson appeared in Troy police court on weapons charges. Meanwhile, a fourth suspect was released to the custody of his parents due to his age.\n\nPolice say 19-year-old Ahziarh Carter was killed and a 14-year-old injured, during a shooting Monday night on Hutton Avenue.\n\nThe officers saw a car drive off from the scene, leading police on a chase through Rensselaer and Albany Counties.\n\nThe car was stopped in Colonie, police say it had the four young men inside, along with multiple weapons.\n\n"There were three firearms located within the vehicle. There was one located outside the vehicle, that we certainly feel came from that vehicle. We'll be doing tests on those to see if they were used in any other shootings throughout the city or possibly the Capital District or elsewhere," said Troy Police Chief John Tedesco.\n\nThe 14-year-old shooting victim's name has not been released. The teen is expected to survive.<|endoftext|>Pittsburgh has a first dog.\n\nMayor Bill Peduto adopted 3-year-old Lilly, a Boston terrier, on Wednesday from the Western Pennsylvania Humane Society. Lilly made her first appearance in City Hall in the afternoon accompanied by the mayor’s girlfriend, Caitlin Lasky.\n\n“I think she likes Caitlin more than me, so I’ve got some work to do,” Peduto said, adding that Lasky would care for Lilly when he’s busy with city affairs.\n\nThe mayor hopes to bring Lilly to work with him rather than letting her stay at his Point Breeze home in a crate during the day.\n\n“She’s a very, very quiet dog and she’ll be able to sit during the day,” Peduto said. “It wouldn’t be the first time. Back in the day, Mayor (Tom) Murphy had a cat. I don’t know if there had ever been a dog, though.”\n\nHe said he’d been looking for a Boston terrier, and Lilly popped up “out of the blue” at the Humane Society. An elderly Pittsburgh woman, who could no longer take care of Lilly, took her to the shelter for adoption.\n\nHumane Society spokeswoman Kristen Lane described Lilly as a “special needs” dog because she is partially blind. Peduto said he intends to take the dog to a specialist for treatment.\n\n“She’s a ‘Burgh dog. She’s a Lilly, and she’s really sweet,” Lane said.\n\nThe mayor paid $155 for Lilly, the shelter’s typical fee, which includes sterilization, a microchip for identification, a health checkup, all necessary shots and deworming and flea treatment, Lane said.\n\nBob Bauder is a staff writer for Trib Total Media. He can be reached at 412-765-2312 or bbauder@tribweb.com.<|endoftext|>This week we journey through the looking glass into a world of very bizarre logic.\n\nThe Thrullg is Great!\n\nI don’t know how I never saw it before but the Thrullg is apparently actually quite good. He’s definitely better than the Trollkin Sorcerer and even though he costs three times as much you should take him instead of the Sorcerer every time for upkeep removal.
    \end{Verbatim}
    \end{tcolorbox}
    \begin{tcolorbox}[title={Duo w/ DCD}, colback=green!5, colframe=green!50!black, boxrule=0.4pt, arc=1mm, left=3pt, right=3pt, top=1pt, bottom=1pt, boxsep=1pt]
    \begin{Verbatim}[breaklines=true, breakanywhere=true]
    \n\nThe Thrullg’s Stamp\n\n\nHe’s actually very great!\n\nBut\n\nn’t<|endoftext|>
    \end{Verbatim}
    \end{tcolorbox}
    \begin{tcolorbox}[title={TTCD (w/ Shortcut)}, colback=purple!5, colframe=purple!60!black, boxrule=0.4pt, arc=1mm, left=3pt, right=3pt, top=1pt, bottom=1pt, boxsep=1pt]
    \begin{Verbatim}[breaklines=true, breakanywhere=true]
    Immediately, the he’ll measure one of the because funny things made in the years should The Th at the Sorcerer’s news that works,<|endoftext|>
    \end{Verbatim}
    \end{tcolorbox}
    \caption{Prefix-conditioned generated sample from TTCD (w/ Shortcut) and Duo w/ DCD for 4 step generation on canvas of 32 tokens}
    \label{fig:qualitative_distilled_app}
\end{figure}

\begin{figure}[H]
    \centering
    \scriptsize
    \begin{tcolorbox}[title={Prompt}, colback=blue!5, colframe=blue!60!black, boxrule=0.4pt, arc=1mm, left=3pt, right=3pt, top=1pt, bottom=1pt, boxsep=1pt]
    \begin{Verbatim}[breaklines=true, breakanywhere=true]
    <|endoftext|>Freedom” is not an abstract concept, relegated to ancient history books on a dusty shelf. It is the very tangible ability to think, to speak, to act and do without anyone saying I cannot, so long as my doing so does not interfere with my neighbor’s ability to do the same. When Vermonters remember that, we’ll recognize it is time to end the failed policy of prohibition by legalizing, taxing and regulating marijuana consumption.<|endoftext|>Free E-book Reader Headed To Nintendo 3DS In Japan\n\nBy Sato . July 6, 2013 . 3:00pm\n\nDai Nippon Printing, a major Japanese printing company, revealed during the Tokyo International Book Fair (which is currently taking place) that they and Nintendo will be introducing a new e-book reader geared for children this Fall for the Nintendo 3DS.\n\nAccording to the printing company, they have realized that Nintendo 3DS users consist of many grade-school children, making it a perfect device to introduce simple-to-read ebooks for their young audience.\n\nIn recent years, they’ve been heavily emphasizing on their e-book business, and are expecting to expand their market with Nintendo, who they are confident will guide them in the right direction with the Nintendo 3DS.\n\nThe downloadable software, known as honto, will be available for free sometime this Fall. It will feature a colorful design for the shop page, and a simple navigational system. The software is expected to sell novels for children, picture books, and various study material, which will also be categorized according to the reader’s age and kanji reading level. All the ebooks from honto will have font size adjusting options and will fully utilize both screens on the Nintendo 3DS, which can be read horizontally, like an actual book.\n\nDai Nippon Printing’s product appeal for this upcoming 3DS software is that they will be providing services using devices that many already own, in the 3DS and 3DS XL, so that parents can rest assured and won’t need to worry about not owning tablets or smart phones.\n\nThe price of the ebooks will be much cheaper than their actual printed counterparts, too, and will range between 700 to 2,000 yen.\n\nHonto is expcted to be released this Fall for Nintendo 3DS.<|endoftext|>UFC champ Conor McGregor is making all kinds of predictions ahead of his UFC 205 main event showdown with Eddie Alvarez.\n\nMcGregor (20-3 MMA, 8-1 UFC), the UFC featherweight champ who will challenge Alvarez (28-4 MMA, 3-1 UFC) for lightweight gold in the UFC 205 headliner, has already boldly forecasted a first-round knockout for the fight. Now he’s made his guess for pay-per-view numbers, and once again “The Notorious” said he’s expecting to set records.\n\n“I’m almost certain we’re talking the 2-million mark,” McGregor told MMAjunkie. “I feel it will break the 2-million mark. That’s what I’d like, and then go from there.”\n\nUFC 205 takes place Nov. 12 at Madison Square Garden in New York City. The main card airs on pay-per-view following prelims on FS1 and UFC Fight Pass.\n\nMcGregor, No. 1 in the latest USA TODAY Sports/MMAjunkie MMA featherweight rankings, said UFC 205 has all the necessary ingredients to draw a historic buy rate for an MMA event. His most recent fight, a majority decision victory over Nate Diaz at UFC 202 in August, reportedly broke the UFC pay-per-view record with 1.65 million buys.\n\nFrom McGregor’s perspective, UFC 202, which marked a rematch of his only octagon loss against Diaz at UFC 196 in March, was not given nearly the promotional support of his upcoming fight with Alvarez.\n\nMoreover, McGregor is aware of the historic nature of the event. Not only is he attempting to become the first simultaneous two-division titleholder for the UFC at the expense of No. 1-ranked lightweight Alvarez, but there are also two other championship fights scheduled as well as an abundance of other quality    
    \end{Verbatim}
    \end{tcolorbox}
    \begin{tcolorbox}[title={Duo w/ DCD}, colback=green!5, colframe=green!50!black, boxrule=0.4pt, arc=1mm, left=3pt, right=3pt, top=1pt, bottom=1pt, boxsep=1pt]
    \begin{Verbatim}[breaklines=true, breakanywhere=true]
    UFC fighters McGregor would be interested in, no matter how slim and incomprehensible the competition.\n\nHowever, UFC 205 himself McGregor that McGregor odds had peaked in the second of April 4 UFC 206, McGregor making it 59-1 in that fight\n\n“FinThis Alvarez, and Conor McGregor, it’s been good to right in the air for UFC 205, is only entertaining,” McGregor said\n\n“For the most part, the game to the fight McGregor-like is positioning, positioning, and positioning, and combat. He will trying to stand out as his opponent of Conor champion.\n\n<|endoftext|>
    \end{Verbatim}
    \end{tcolorbox}
    \begin{tcolorbox}[title={TTCD (w/ Shortcut)}, colback=purple!5, colframe=purple!60!black, boxrule=0.4pt, arc=1mm, left=3pt, right=3pt, top=1pt, bottom=1pt, boxsep=1pt]
    \begin{Verbatim}[breaklines=true, breakanywhere=true]
    in multiple matches, including Chris Channo and Nate Diaz, who rounds in as they began that fight on headline loss for Tony…r. Ultimate athlete is enjoying so, what I’d like to watch and keep them fans and the team, and a little again<|endoftext|>
    \end{Verbatim}
    \end{tcolorbox}
    \caption{Prefix-conditioned generated sample from TTCD (w/ Shortcut) and Duo w/ DCD for 4 step generation on canvas of 128 tokens}
    \label{fig:qualitative_distilled_app}
\end{figure}

%% file: checklist.tex
\section*{NeurIPS Paper Checklist}



\begin{enumerate}

\item {\bf Claims}
    \item[] Question: Do the main claims made in the abstract and introduction accurately reflect the paper's contributions and scope?
    \item[] Answer: \answerYes{} 
    \item[] Justification: All claims in the abstract are accurately represented in the paper

\item {\bf Limitations}
    \item[] Question: Does the paper discuss the limitations of the work performed by the authors?
    \item[] Answer: \answerYes{} 
    \item[] Justification: We have Sec~\ref{sec:conclusion} in the paper which discusses our limitations

\item {\bf Theory assumptions and proofs}
    \item[] Question: For each theoretical result, does the paper provide the full set of assumptions and a complete (and correct) proof?
    \item[] Answer: \answerYes{} 
    \item[] Justification: We only have one lemma in the paper and we prove that in Sec~\ref{subsec:proof_of_lemma}.

    \item {\bf Experimental result reproducibility}
    \item[] Question: Does the paper fully disclose all the information needed to reproduce the main experimental results of the paper to the extent that it affects the main claims and/or conclusions of the paper (regardless of whether the code and data are provided or not)?
    \item[] Answer: \answerYes{} 
    \item[] Justification: Yes, all details are shared. Work is reproducible. 

\item {\bf Open access to data and code}
    \item[] Question: Does the paper provide open access to the data and code, with sufficient instructions to faithfully reproduce the main experimental results, as described in supplemental material?
    \item[] Answer: \answerNo{} 
    \item[] Justification: The dataset is open-source. We use open-source repository for experimentation and can be done easily. 

\item {\bf Experimental setting/details}
    \item[] Question: Does the paper specify all the training and test details (e.g., data splits, hyperparameters, how they were chosen, type of optimizer) necessary to understand the results?
    \item[] Answer: \answerYes{} 
    \item[] Justification: We follow prior work ~\cite{sahoo2024simple} for our main results. 

\item {\bf Experiment statistical significance}
    \item[] Question: Does the paper report error bars suitably and correctly defined or other appropriate information about the statistical significance of the experiments?
    \item[] Answer: \answerNo{} 
    \item[] Justification: The standard errors are small for all the numbers reported. We have computed standard deviation over three seeds. The reported numbers in paper are average over three seeds. For ease of exposition error bars are omitted. 

\item {\bf Experiments compute resources}
    \item[] Question: For each experiment, does the paper provide sufficient information on the computer resources (type of compute workers, memory, time of execution) needed to reproduce the experiments?
    \item[] Answer: \answerYes{} 
    \item[] Justification: Yes, we have appendix Sec~\ref{app:compute}.
    
\item {\bf Code of ethics}
    \item[] Question: Does the research conducted in the paper conform, in every respect, with the NeurIPS Code of Ethics \url{https://neurips.cc/public/EthicsGuidelines}?
    \item[] Answer: \answerYes{} 
    \item[] Justification: We conform to the code of ethics. 

\item {\bf Broader impacts}
    \item[] Question: Does the paper discuss both potential positive societal impacts and negative societal impacts of the work performed?
    \item[] Answer: \answerNA{} 
    \item[] Justification: Our work advances language models and there are no additional malicious use-cases to the best of our knowledge, beyond using better language models for harm. 
    
\item {\bf Safeguards}
    \item[] Question: Does the paper describe safeguards that have been put in place for responsible release of data or models that have a high risk for misuse (e.g., pre-trained language models, image generators, or scraped datasets)?
    \item[] Answer: \answerNo{} 
    \item[] Justification: No such risks

\item {\bf Licenses for existing assets}
    \item[] Question: Are the creators or original owners of assets (e.g., code, data, models), used in the paper, properly credited and are the license and terms of use explicitly mentioned and properly respected?
    \item[] Answer: \answerNA{} 
    \item[] Justification: All the data used is open-source and appropriate citations are given. 

\item {\bf New assets}
    \item[] Question: Are new assets introduced in the paper well documented and is the documentation provided alongside the assets?
    \item[] Answer: \answerNA{} 
    \item[] Justification: No new assets 

\item {\bf Crowdsourcing and research with human subjects}
    \item[] Question: For crowdsourcing experiments and research with human subjects, does the paper include the full text of instructions given to participants and screenshots, if applicable, as well as details about compensation (if any)? 
    \item[] Answer: \answerNA{} 
    \item[] Justification: No crowdsourced research

\item {\bf Institutional review board (IRB) approvals or equivalent for research with human subjects}
    \item[] Question: Does the paper describe potential risks incurred by study participants, whether such risks were disclosed to the subjects, and whether Institutional Review Board (IRB) approvals (or an equivalent approval/review based on the requirements of your country or institution) were obtained?
    \item[] Answer: \answerNA{} 
    \item[] Justification: No such research 

\item {\bf Declaration of LLM usage}
    \item[] Question: Does the paper describe the usage of LLMs if it is an important, original, or non-standard component of the core methods in this research? Note that if the LLM is used only for writing, editing, or formatting purposes and does \emph{not} impact the core methodology, scientific rigor, or originality of the research, declaration is not required.
    \item[] Answer: \answerNA{} 
    \item[] Justification: LLMs are not used. 

\end{enumerate}

%% file: neurips_2026.bbl
\begin{thebibliography}{10}

\bibitem{austin2021structured}
Jacob Austin, Daniel~D Johnson, Jonathan Ho, Daniel Tarlow, and Rianne Van Den~Berg.
\newblock Structured denoising diffusion models in discrete state-spaces.
\newblock {\em Advances in neural information processing systems}, 34:17981--17993, 2021.

\bibitem{avdeyev2023dirichlet}
Pavel Avdeyev, Chenlai Shi, Yuhao Tan, Kseniia Dudnyk, and Jian Zhou.
\newblock Dirichlet diffusion score model for biological sequence generation.
\newblock In {\em International Conference on Machine Learning}, pages 1276--1301. PMLR, 2023.

\bibitem{boffi2025build}
Nicholas~M Boffi, Michael~S Albergo, and Eric Vanden-Eijnden.
\newblock How to build a consistency model: Learning flow maps via self-distillation.
\newblock {\em arXiv preprint arXiv:2505.18825}, 2025.

\bibitem{davis2024fisher}
Oscar Davis, Samuel Kessler, Mircea Petrache, {\.I}smail~{\.I} Ceylan, Michael Bronstein, and Avishek~J Bose.
\newblock Fisher flow matching for generative modeling over discrete data.
\newblock {\em Advances in Neural Information Processing Systems}, 37:139054--139084, 2024.

\bibitem{dieleman2022continuous}
Sander Dieleman, Laurent Sartran, Arman Roshannai, Nikolay Savinov, Yaroslav Ganin, Pierre~H Richemond, Arnaud Doucet, Robin Strudel, Chris Dyer, Conor Durkan, et~al.
\newblock Continuous diffusion for categorical data.
\newblock {\em arXiv preprint arXiv:2211.15089}, 2022.

\bibitem{frans2024one}
Kevin Frans, Danijar Hafner, Sergey Levine, and Pieter Abbeel.
\newblock One step diffusion via shortcut models.
\newblock {\em arXiv preprint arXiv:2410.12557}, 2024.

\bibitem{gulrajani2023likelihood}
Ishaan Gulrajani and Tatsunori~B Hashimoto.
\newblock Likelihood-based diffusion language models.
\newblock {\em Advances in Neural Information Processing Systems}, 36:16693--16715, 2023.

\bibitem{jocontinuous}
Jaehyeong Jo and Sung~Ju Hwang.
\newblock Continuous diffusion model for language modeling.
\newblock In {\em The Thirty-ninth Annual Conference on Neural Information Processing Systems}.

\bibitem{kim2025train}
Jaeyeon Kim, Kulin Shah, Vasilis Kontonis, Sham Kakade, and Sitan Chen.
\newblock Train for the worst, plan for the best: Understanding token ordering in masked diffusions.
\newblock {\em arXiv preprint arXiv:2502.06768}, 2025.

\bibitem{lee2026one}
Chanhyuk Lee, Jaehoon Yoo, Manan Agarwal, Sheel Shah, Jerry Huang, Aditi Raghunathan, Seunghoon Hong, Nicholas~M Boffi, and Jinwoo Kim.
\newblock One-step language modeling via continuous denoising.
\newblock {\em arXiv preprint arXiv:2602.16813}, 2026.

\bibitem{li2022diffusion}
Xiang Li, John Thickstun, Ishaan Gulrajani, Percy~S Liang, and Tatsunori~B Hashimoto.
\newblock Diffusion-lm improves controllable text generation.
\newblock {\em Advances in neural information processing systems}, 35:4328--4343, 2022.

\bibitem{lovelace2026stop}
Justin Lovelace, Christian Belardi, Sofian Zalouk, Adhitya Polavaram, Srivatsa Kundurthy, and Kilian~Q Weinberger.
\newblock Stop-think-autoregress: Language modeling with latent diffusion planning.
\newblock {\em arXiv preprint arXiv:2602.20528}, 2026.

\bibitem{lovelace2023latent}
Justin Lovelace, Varsha Kishore, Chao Wan, Eliot Shekhtman, and Kilian~Q Weinberger.
\newblock Latent diffusion for language generation.
\newblock {\em Conference on Neural Information Processing Systems (NeurIPS)}, 2023.

\bibitem{nie2025large}
Shen Nie, Fengqi Zhu, Zebin You, Xiaolu Zhang, Jingyang Ou, Jun Hu, Jun Zhou, Yankai Lin, Ji-Rong Wen, and Chongxuan Li.
\newblock Large language diffusion models.
\newblock {\em arXiv preprint arXiv:2502.09992}, 2025.

\bibitem{peebles2023scalable}
William Peebles and Saining Xie.
\newblock Scalable diffusion models with transformers.
\newblock In {\em Proceedings of the IEEE/CVF international conference on computer vision}, pages 4195--4205, 2023.

\bibitem{pynadath2025candi}
Patrick Pynadath, Jiaxin Shi, and Ruqi Zhang.
\newblock Candi: Hybrid discrete-continuous diffusion models.
\newblock {\em arXiv preprint arXiv:2510.22510}, 2025.

\bibitem{sahoo2024simple}
Subham Sahoo, Marianne Arriola, Yair Schiff, Aaron Gokaslan, Edgar Marroquin, Justin Chiu, Alexander Rush, and Volodymyr Kuleshov.
\newblock Simple and effective masked diffusion language models.
\newblock {\em Advances in Neural Information Processing Systems}, 37:130136--130184, 2024.

\bibitem{sahoo2025diffusion}
Subham~Sekhar Sahoo, Justin Deschenaux, Aaron Gokaslan, Guanghan Wang, Justin Chiu, and Volodymyr Kuleshov.
\newblock The diffusion duality.
\newblock {\em arXiv preprint arXiv:2506.10892}, 2025.

\bibitem{salimans2022progressive}
Tim Salimans and Jonathan Ho.
\newblock Progressive distillation for fast sampling of diffusion models.
\newblock {\em arXiv preprint arXiv:2202.00512}, 2022.

\bibitem{schiff2025simple}
Yair Schiff, Subham~Sekhar Sahoo, Hao Phung, Guanghan Wang, Alexander Rush, Volodymyr Kuleshov, Hugo Dalla-Torre, Sam Boshar, Bernardo~P de~Almeida, and Thomas Pierrot.
\newblock Simple guidance mechanisms for discrete diffusion models.
\newblock In {\em ... International Conference on Learning Representations}, volume 2025, page 44153, 2025.

\bibitem{shah2024causal}
Kulin Shah, Nishanth Dikkala, Xin Wang, and Rina Panigrahy.
\newblock Causal language modeling can elicit search and reasoning capabilities on logic puzzles.
\newblock {\em Advances in Neural Information Processing Systems}, 37:56674--56702, 2024.

\bibitem{stark2024dirichlet}
Hannes Stark, Bowen Jing, Chenyu Wang, Gabriele Corso, Bonnie Berger, Regina Barzilay, and Tommi Jaakkola.
\newblock Dirichlet flow matching with applications to dna sequence design.
\newblock In {\em International Conference on Machine Learning}, pages 46495--46513. PMLR, 2024.

\bibitem{wu2025fast}
Chengyue Wu, Hao Zhang, Shuchen Xue, Zhijian Liu, Shizhe Diao, Ligeng Zhu, Ping Luo, Song Han, and Enze Xie.
\newblock Fast-dllm: Training-free acceleration of diffusion llm by enabling kv cache and parallel decoding.
\newblock {\em arXiv preprint arXiv:2505.22618}, 2025.

\bibitem{ye2025dream}
Jiacheng Ye, Zhihui Xie, Lin Zheng, Jiahui Gao, Zirui Wu, Xin Jiang, Zhenguo Li, and Lingpeng Kong.
\newblock Dream 7b: Diffusion large language models.
\newblock {\em arXiv preprint arXiv:2508.15487}, 2025.

\bibitem{ye2023dinoiser}
Jiasheng Ye, Zaixiang Zheng, Yu~Bao, Lihua Qian, and Mingxuan Wang.
\newblock Dinoiser: Diffused conditional sequence learning by manipulating noises.
\newblock {\em arXiv preprint arXiv:2302.10025}, 2023.

\bibitem{zheng2025continuously}
Huangjie Zheng, Shansan Gong, Ruixiang Zhang, Tianrong Chen, Jiatao Gu, Mingyuan Zhou, Navdeep Jaitly, and Yizhe Zhang.
\newblock Continuously augmented discrete diffusion model for categorical generative modeling.
\newblock {\em arXiv preprint arXiv:2510.01329}, 2025.

\bibitem{zheng2024masked}
Kaiwen Zheng, Yongxin Chen, Hanzi Mao, Ming-Yu Liu, Jun Zhu, and Qinsheng Zhang.
\newblock Masked diffusion models are secretly time-agnostic masked models and exploit inaccurate categorical sampling.
\newblock {\em arXiv preprint arXiv:2409.02908}, 2024.

\bibitem{zhou2025coevolutionary}
Cai Zhou, Chenxiao Yang, Yi~Hu, Chenyu Wang, Chubin Zhang, Muhan Zhang, Lester Mackey, Tommi Jaakkola, Stephen Bates, and Dinghuai Zhang.
\newblock Coevolutionary continuous discrete diffusion: Make your diffusion language model a latent reasoner.
\newblock {\em arXiv preprint arXiv:2510.03206}, 2025.

\end{thebibliography}
